\documentclass[11pt]{article}
\usepackage{psfig}
\usepackage{amsmath,amsthm}
\usepackage{geometry}                
\usepackage{graphicx}
\usepackage{amssymb}
\usepackage{epstopdf}
\usepackage{mathtools}
\usepackage{framed}
\usepackage{algorithmic}
\usepackage{algorithm}
\usepackage{mathrsfs}

\usepackage{epsfig}
\usepackage{url}

\begin{document}

\fontsize{11}{13pt}\selectfont

%
%
%
%
%
%
%

\title{Reliable Clustering of Bernoulli Mixture Models}
\author{
Amir~Najafi
\thanks{E-mails: najafy@ce.sharif.edu~,~motahari@sharif.edu.}
\and
Abolfazl~Motahari
\footnotemark[1]
\and
Hamid~R.~Rabiee
\thanks{Email: rabiee@sharif.edu.}
}

\date{}
\maketitle

\vspace*{-2mm}
\begin{small}
\begin{center}
$^*$ Bioinformatics Research Laboratory (BRL),\\
$^\dagger$ Digital Media Laboratory (DML),\\
Department of Computer Engineering,\\
Sharif University of Technology, Tehran, Iran
\end{center}
\end{small}
\vspace*{6mm}

\newtheorem{thm}{Theorem}
\newtheorem{corl}{Corollary}
\newtheorem{note}{Note}
\newtheorem{lemma}{Lemma}
\newtheorem{definition}{Definition}

\begin{abstract}
A Bernoulli Mixture Model (BMM) is a finite mixture of random binary vectors with independent dimensions. The problem of clustering BMM data arises in a variety of real-world applications, ranging from population genetics to activity analysis in social networks. In this paper, we analyze the clusterability of BMMs from a theoretical perspective, when the number of clusters is unknown. In particular, we stipulate a set of conditions on the sample complexity and dimension of the model in order to guarantee the Probably Approximately Correct (PAC)-clusterability of a dataset. To the best of our knowledge, these findings are the first non-asymptotic bounds on the sample complexity of learning or clustering BMMs.
\end{abstract}

\section{Introduction}
\label{sec:intro}
Demixing data samples from mixture models, also called model-based clustering, has long been studied by statisticians and computer scientists. Although plenty of promising algorithms have been introduced in this area, see \cite{hollander2013nonparametric,bouveyron2014model,mcnicholas2016model,muller2015bayesian}, fewer efforts have been focused on deriving theoretical guarantees on reliable clustering of data samples. The aim of this paper is to elaborate on this shortcoming by deriving analytic guarantees on the clusterability of a particular case of interest: Bernoulli Mixture Models (BMM). 

A Bernoulli Model (BM) refers to a random binary vector $\boldsymbol{X}=\left[X_1,\ldots,X_L\right]\in\left\{0,1\right\}^L$ with independent random components, where $L$ denotes the model dimension and each $X_i$ is a Bernoulli random variable with success probability (or frequency) $p_i$, i.e. $X_i\sim\mathrm{Bern}\left(p_i\right)$. Let us define ${\boldsymbol{p}}:=\left[p_1,\ldots,p_L\right]\in\left[0,1\right]^{L}$. Then, $\mathbb{P}_{\mathrm{BM}}\left(\boldsymbol{X};{\boldsymbol{p}}\right)$ denotes the probability distribution of a Bernoulli model with frequency vector ${\boldsymbol{p}}$:
\begin{equation*}
\mathbb{P}_{\mathrm{BM}}\left(\boldsymbol{X};{\boldsymbol{p}}\right):=
\prod_{\ell=1}^{L}{p_{\ell}}^{X_{\ell}}\left(1-p_{\ell}\right)^{1-X_{\ell}}.
\end{equation*}
In this regard, a BMM is defined as a mixture of a finite number of Bernoulli models \cite{juan2004bernoulli}. Mathematically speaking, the probability distribution of a BMM can be expressed as
\begin{equation}
\mathbb{P}_{\mathrm{BMM}}\left(\boldsymbol{X};K,{\boldsymbol{p}}^{\left(1\right)},\ldots,{\boldsymbol{p}}^{\left(K\right)},\boldsymbol{w}\right)
=\sum_{k=1}^{K}w_k \mathbb{P}_{\mathrm{BM}}\left(\boldsymbol{X};{\boldsymbol{p}}^{\left(k\right)}\right),
\label{eq:mixture}
\end{equation}
where $K\in\mathbb{N}$ denotes the number of mixture components (or clusters), $\left\{{\boldsymbol{p}}^{\left(1\right)},\ldots,{\boldsymbol{p}}^{\left(K\right)}\right\}$ is the set of frequency vectors associated to mixture components, and $\boldsymbol{w}=\left(w_1,\ldots,w_K\right)$ is the mixture weight vector with $\sum_{k}w_k=1$, and $w_k\ge0$.
Let $\boldsymbol{P}$ be a $K\times L$ frequency matrix with ${\boldsymbol{p}}^{\left(k\right)}$ as its $k$th row.
For simplicity, we denote $\mathcal{B}=\mathcal{B}\left(K,\boldsymbol{P},\boldsymbol{w}\right)$ as the BMM with the above-mentioned parameters and specifications. Let $\boldsymbol{X}_1,\ldots,\boldsymbol{X}_n\in\left\{0,1\right\}^L$ be $n$ i.i.d. sample vectors drawn from $\mathcal{B}$. We define $\boldsymbol{\mathcal{X}}$ as the matrix $\left[\boldsymbol{X}_1\vert\ldots\vert \boldsymbol{X}_n\right]^T\in\left\{0,1\right\}^{n\times L}$ to represent a given dataset. 

The problem that we have tackled in this paper is the clustering of rows of $\boldsymbol{\mathcal{X}}$, such that clusters with probability at least $1-\zeta$ are approximately (up to an $\epsilon$ fraction of mis-clustering) correct, for arbitrarily small $\epsilon,\zeta>0$. In order to guarantee the information-theoretic possibility of such clustering, i.e. without considering the required computational cost, we establish novel bounds in the form of $n\ge\mathrm{poly}\left(\epsilon^{-1},\zeta^{-1}\right)$ and $L>\mathrm{poly}\left(\epsilon^{-1}\right)$, where $\mathrm{poly}$ refers to a polynomial function. It should be noted that $K$, $\boldsymbol{P}$ and $\boldsymbol{w}$ are all assumed to be unknown. Statisticians have been studying BMMs for a long time \cite{tiedeman1955study,wolfe1970pattern,baker1998distributional,biernacki1999improvement}. However, PAC-learnability (or PAC-clusterability)\footnote{In this paper, the notion of PAC-learnability is used in the information-theoretic sense, and to address those learning tasks that can be learned with a polynomial sample complexity w.r.t. $\epsilon$ and $\zeta$. This notation is consistent with that of \cite{mohri2012foundations}. For those cases where the learning algorithm has also a polynomial computational complexity, the term ``efficiently PAC-learnable" has been used.} of BMMs in terms of the minimum required sample size and/or model dimension has remained an open problem. To the best of our knowledge, this paper is the first attempt toward this goal by deriving a set of non-asymptotic conditions under which reliable clustering is possible. 

The paper is organized as follows: In Section \ref{sec:relatedWorks}, related works are discussed. Section \ref{sec:notation} formally presents our main results, where proofs and further discussions are given in Section \ref{sec:proof}.  Finally, conclusions are made in Section \ref{sec:conc}.

\section{Related Works}
\label{sec:relatedWorks}

Employment of BMMs in order to model multi-dimensional categorical data goes back to \cite{lazarsfeld1968latent}, while more detailed mathematical and historical explanations can be found, for example,  in \cite{bishop2006pattern,mclachlan2004finite,li2016conditional}. In two classic works \cite{biernacki1999improvement} and \cite{celeux1996entropy}, a series of heuristic measures have been introduced to assess the number of mixture components in a BMM; However, their performance is validated only through experimental investigations. Authors of these papers have {\it {conjectured}} that learnability is possible as long as independence holds between cluster parameters, while their studies lack a theoretical sufficiency analysis. From an algorithmic point of view, the Expectation-Maximization (EM) algorithm is the most widely used framework for statistical inference in BMMs (see \cite{li2016conditional} and \cite{palmer2016generalized}). In \cite{figueiredo2002unsupervised}, a popular EM-based technique for unsupervised learning of finite mixture models (including BMMs) is introduced, which makes no assumption on the number of mixture components. Also, see \cite{fraley2002model} for another well-cited paper on model-based clustering of mixture model data. In \cite{juan2004initialisation}, EM is employed for parameter initialization in a number of existing inference algorithms on BMMs. From a theoretical perspective, a set of statistical guarantees on the convergence of EM algorithm in mixture model problems has been recently given in \cite{balakrishnan2017statistical}, however, authors have mainly focused on Gaussian distributions rather than Bernoulli models. So far, theoretical analysis of Gaussian Mixture Models (GMM) have been more successful compared to many discrete mixture models \cite{diakonikolas2016learning,kalai2016disentangling}. This might be due to both the continuous nature, and also the more-favorable analytic form of Gaussian distributions \cite{chan2014efficient}. Recently, nearly tight lower and upper-bounds on the sample complexity of learning GMM distributions have been derived \cite{ashtiani2018nearly}.

Our work is also related to Bayesian non-parametric approaches in the sense that the number of clusters is open-ended, and will be inferred based on the observed data. Some good reviews on non-parametric approaches in statistics can be found, for example, in \cite{rousseau2016frequentist,orbanz2011bayesian,gershman2012tutorial}. In particular, \cite{teh2005sharing} has proposed a unified non-parametric framework for model-based clustering with the use of hierarchical Dirichlet mixtures. All of the studies on BMMs that we have been reviewed so far share a common property: at their best, they only prove convergence to a sub-optimal likelihood value, rather than providing guarantees on the accuracy of the final clustering/learning. Also, sample complexity lower-bounds, i.e. minimum required sample size $n$ or dimension $L$, for either reliable learning or clustering of BMMs is still an open problem.

From a geneticist's point of view, this paper basically builds upon the statistical model presented in \cite{pritchard2000inference}. Based on this model, the genetic sequence of an individual from a particular specie can be represented with a binary vector, where each dimension denotes the {\it {absence}} or {\it {presence}} of a certain genetic variant. In addition, for the majority of cases, dimensions can be assumed to be statistically independent from each other. According to \cite{pritchard2000inference}, one can model the effect of population inhomogeneity (the presence of population mixtures in a biological dataset) via BMMs.  More discussions on this issue can be found, for example, in \cite{visscher2012five}, \cite{pella2006gibbs} and \cite{kopelman2015clumpak}. In \cite{evanno2005detecting}, authors have performed a simulation study based on \cite{pritchard2000inference}, in order to assess the number of clusters in a given population. A number of software packages for computational population analysis can be found in \cite{kopelman2015clumpak,catchen2013stacks,peakall2006genalex,purcell2007plink}, which mainly focus on binary datasets, the same configuration that we have considered in this paper. Our problem setting encompasses both models described in \cite{pritchard2000inference} and \cite{falush2003inference}, since we make no restrictive assumptions, such as {\it {independence}}, on the latent frequencies. Recently, Genome-Wide Association Studies (GWAS) which involve associating human diseases to genetic variants have gained huge popularity. The role of population stratification in GWAS, an important application of BMMs in genetics research, is discussed in \cite{visscher2012five} and \cite{pritchard2000association}. For more research on the employment of BMMs in GWA studies, see \cite{yu2006unified,zhou2017fast,najafi2017statistical,price2006principal}.

A critical issue that needs to be discussed here is the following fundamental question: when is a BMM guaranteed to be identifiable? We call a BMM identifiable whenever there is a unique set of parameters $\left(K,\boldsymbol{P},\boldsymbol{w}\right)$ that corresponds to its probability measure. Reliable clustering of BMMs seems meaningless if there exist more than one true generative models for the samples. Identifiability of BMMs has been addressed in \cite{gyllenberg1994non}, where authors have shown that BMMs cannot be {\it {strictly}} identifiable regardless of their dimension, meaning that there always exist some sets of parameters which result in the same probability distribution. However, it does not mean that for every setting $\left(K,\boldsymbol{P},\boldsymbol{w}\right)$ there must exist another parameter set to produce the same model. Motivated by this idea, in \cite{carreira2000practical} authors have investigated practical identifiability of BMMs via computer simulations. In \cite{allman2009identifiability}, it has been proved that for $L\ge 2\left\lceil \log_2 K \right\rceil+1$, BMMs become {\it {generically}} identifiable meaning that sets of parameters with the same probability distributions have a zero Lebesgue measure in the space of parameters. Therefore, we can only focus on identifiable cases without any loss of generality for our results as long as the above condition holds. More precisely, the conditions that we stipulate in this paper to guarantee the clusterability of BMMs also satisfy the identifiability condition, since we guarantee a stronger property that encompasses identifiability.


\section{Main Result}
\label{sec:notation}

In this section, we state our main result and explain its implications. Recall that the problem is to reliably cluster a set of i.i.d. samples which are drawn from a BMM with unknown parameters (the number of clusters is also assumed to be unknown). The samples are embedded as rows of the matrix $\boldsymbol{\mathcal{X}}$. In fact, we design an algorithm which outputs a vector $\boldsymbol{Z}\in\left\{1,2,\ldots\right\}^n$ in which the $i$th element $Z_i$ represents the cluster index of the data sample $\boldsymbol{X}_i$. Ultimately, we compare the output of the algorithm with the true clustering which is denoted by  $\boldsymbol{Z}_T\in\left\{1,2,\ldots,K\right\}^n$.  In this regard, let us state two definitions in order to make the comparison mathematically concrete. 
\\[-3mm]
\begin{definition}[$\epsilon$-pureness]
A selected row sub-matrix of $\boldsymbol{\mathcal{X}}$ is $\epsilon$-pure if at least $1-\epsilon$ fraction of its rows have the same index in $\boldsymbol{Z}_T$.
\end{definition}
\vspace*{1mm}
\begin{definition}[$\epsilon$-correctness]
A clustering algorithm is $\epsilon$-correct on $\boldsymbol{\mathcal{X}}$ if all the output clusters are $\epsilon$-pure.
\end{definition}

Obviously, for a reliable clustering to be feasible, mixture components of the underlying BMM need to be sufficiently far apart from each other. For instance, if a BMM contains two mixture components with exactly the same frequency vectors, no algorithm can index the samples correctly. Therefore, we impose a natural restriction on the parameters of the BMM which makes the clustering a feasible task.
\\[-3mm]
\begin{definition}[ $\left(\mathcal{L},\delta\right)$-separability]
A frequency matrix $\boldsymbol{P}\in\left[0,1\right]^{K\times L}$ is said to be $\left(\mathcal{L},\delta\right)$-separable, if for each pair of rows of $\boldsymbol{P}$, say $k$ and $k'$ with $k\neq k'$, there exist at least $\mathcal{L}\leq L$ column indices $\left\{i_1,\ldots,i_{\mathcal{L}}\right\}\subseteq \left\{1,\ldots,L\right\}$ such that
\begin{equation*}
\left\vert P_{k,i_{\ell}}-P_{k',i_{\ell}} \right\vert \ge \delta,
\quad \ell=1,\ldots,\mathcal{L}.
\end{equation*}
\end{definition}
We may now present our main result in the form of the following theorem which provides a sufficient sample complexity for reliable clustering of BMMs.
\\[-3mm]
\begin{thm}[Non-asymptotic Bounds for Clusterability of Bernoulli Mixture Models]
\label{thm:Main}
Let $\mathcal{B}=\mathcal{B}\left(K,\boldsymbol{P},\boldsymbol{w}\right)$ be a BMM with unknown parameters $K$, $\boldsymbol{P}$ and $\boldsymbol{w}$. However, $\boldsymbol{P}$ is assumed to be $\left(\mathcal{L},\delta\right)$-separable for some $\mathcal{L}\leq L$ and $\delta>0$, and there exists $0<\alpha\leq 1$ such that $w_k\ge\alpha$ for all $k$. Parameters $\mathcal{L},\delta$ and $\alpha$ are assumed to be known. Also, we obviously have $K\leq\left\lceil 1/\alpha\right\rceil$. Let $\boldsymbol{\mathcal{X}}=\left[\boldsymbol{X}_1\vert \boldsymbol{X}_2\vert\ldots\vert \boldsymbol{X}_n\right]^T$ be a dataset including $n$ i.i.d. samples drawn from $\mathcal{B}$. Also, assume $\epsilon,\zeta>0$, such that
\begin{equation*}
\mathcal{L}\ge
\frac{B \log^{3}\left(1/\epsilon\right)}{\epsilon^{2+\frac{1-\alpha}{2\left(\alpha\delta\right)^2}}}
\quad\quad\mathrm{and}\quad\quad
n\ge
\frac{C\log^3\left(1/\epsilon\right)}{\epsilon^{2+\frac{1-\alpha}{2\left(\alpha\delta\right)^2}}}
\log\frac{L}{\zeta},
\end{equation*}
where $B$ and $C$ are constants w.r.t. $\epsilon$ and $\zeta$. Then, there exists a clustering algorithm $\mathscr{A}:\left\{0,1\right\}^{n\times L}\rightarrow\mathbb{N}^n$, such that $\mathscr{A}$ is $\epsilon$-correct on $\boldsymbol{\mathcal{X}}$ with probability at least $1-\zeta$.
\end{thm}

Proof of Theorem \ref{thm:Main} with the mathematical formulation of the constants $B$ and $C$ are given in Section \ref{sec:proof}. Theorem \ref{thm:Main} shows the feasibility of the reliable clustering of samples in $\boldsymbol{\mathcal{X}}$, such that clusters with probability at least $1-\zeta$ are $\epsilon$-pure, for arbitrarily small $\epsilon,\zeta>0$. On the other hand, the imposed conditions on sample size $n$ and the number of {\it {informative dimensions}} $\mathcal{L}$ are $n\ge\mathrm{poly}\left(\epsilon^{-1},\zeta^{-1}\right)$ and $\mathcal{L}>\mathrm{poly}\left(\epsilon^{-1}\right)$, respectively. 

We have already discussed that assuming a minimum deviation among frequency vectors, such as $\left(\mathcal{L},\delta\right)$-separability, is necessary for reliable clustering of data. Similarly, we also need to upper-bound the cluster number $K$, otherwise clustering becomes meaningless. For example, without any condition on $K$, one can always partition a dataset of size $n$ into $n$ distinct clusters where each clusters would be $0$-pure and the clustering is $\epsilon$-correct for any $\epsilon\ge0$.

Once an $\epsilon$-correct clustering is achieved for a dataset $\boldsymbol{\mathcal{X}}$, estimation of frequency matrix $\boldsymbol{P}$ and weight vector $\boldsymbol{w}$ becomes straightforward. In fact, by using Chernoff bound, it is easy to show that each entry of $\boldsymbol{P}$ and each component in $\boldsymbol{w}$ can be estimated with a maximum error of $\epsilon+O\left(n^{-1/2}\right)$ with high probability. However, as mentioned earlier, we only focus on clustering in this paper and thus do not explain estimation of $\boldsymbol{P}$ and $\boldsymbol{w}$ in more details to avoid any distraction.


\subsection{Algorithm}

The proof of theorem \ref{thm:Main} is based on an algorithm which employs a {\it {pureness check}} measure. We propose a new variant of Total Correlation measure (also known as multivariate correlation or multi-information \cite{watanabe1960information,studeny1998multiinformation}) to reliably test whether a given clustering of the dataset $\boldsymbol{\mathcal{X}}$ does include any non $\epsilon$-pure clusters or not. We call this new variant as Maximal Total Correlation (MTC). The core idea for employing such a measure is the following interesting property of BMMs: In a BMM, unlike a single Bernoulli model, different dimensions of the random binary vector are not statistically independent, and thus have positive Mutual Information (MI) w.r.t. each other \cite{li2016conditional}. Total correlation is a natural extension of mutual information which can handle more than two random variables \cite{watanabe1960information}.
\\[-3mm]
\begin{definition}[Total Correlation]
\label{def:DQ}
Assume $\boldsymbol{Q}\in\left\{0,1\right\}^{m\times d}$ to be a row/column sub-matrix of $\boldsymbol{\mathcal{X}}$ (with $m\leq n$ and $d\leq L$). Then, similar to \cite{watanabe1960information} and \cite{studeny1998multiinformation}, the empirical total correlation of $\boldsymbol{Q}$, denoted by $\mathcal{D}\left(\boldsymbol{Q}\right)$, is defined as
\begin{equation}
\mathcal{D}\left(\boldsymbol{Q}\right):=
\mathcal{D}_{\mathrm{KL}}\left(
\hat{\mathbb{P}}_{1,\boldsymbol{Q}}
\Vert
\hat{\mathbb{P}}_{2,\boldsymbol{Q}}
\right),
\end{equation}
where $\hat{\mathbb{P}}_{1,\boldsymbol{Q}}$ denotes the empirical probability distribution underlying the $d$-dimensional rows of $\boldsymbol{Q}$, while $\hat{\mathbb{P}}_{2,\boldsymbol{Q}}$ is defined as:
\begin{equation}
\hat{\mathbb{P}}_{2,\boldsymbol{Q}}\left(\boldsymbol{X}\right)
:=
\prod_{\ell=1}^{d}\hat{p}_{\ell}^{X_{\ell}}\left(1-\hat{p}_{\ell}\right)^{1-X_{\ell}},
\quad
\forall\boldsymbol{X}\in\left\{0,1\right\}^{d},
\end{equation}
 with $\hat{p}_\ell$ being the empirical frequency of the $\ell$th column of $\boldsymbol{Q}$, i.e.
\begin{equation*}
\hat{p}_{\ell}:= \frac{1}{n}\sum_{i=1}^{n}Q_{i,\ell}~,~\ell=1,2,\ldots,d.
\end{equation*}
\end{definition}

In fact, $\mathcal{D}\left(\boldsymbol{Q}\right)$ is the Kullback-Leibler divergence between two distributions obtained under two separate assumptions. Under the first assumption, no restriction is imposed on the origin of the samples and the empirical distribution  $\hat{\mathbb{P}}_{1,\boldsymbol{Q}}$ is simply an estimate of the true underlying distribution  of the rows of $\boldsymbol{Q}$. Under the second assumption, samples are drawn from a single Bernoulli model and $\hat{\mathbb{P}}_{2,\boldsymbol{Q}}$ can be used as another estimate of the true distribution. Using the language of {\it {types}} from information theory, $\hat{\mathbb{P}}_{2,\boldsymbol{Q}}$ is a product of the marginal types along each dimension. Therefore, if the second assumption does hold, then the two distributions become equal as $n$ goes to infinity. In other words, based on the law of large numbers, we have:
\begin{equation*}
\lim_{n\rightarrow\infty}\mathcal{D}\left(\boldsymbol{Q}\right)=\mathcal{D}_{\mathrm{KL}}\left(
\lim_{n\rightarrow\infty}\hat{\mathbb{P}}_{1,\boldsymbol{Q}}
\Vert
\lim_{n\rightarrow\infty}\hat{\mathbb{P}}_{2,\boldsymbol{Q}}
\right)=0.
\end{equation*}
On the other hand, if $\boldsymbol{Q}$ consists of samples from a BMM with a sufficient level of contributions from different mixture components, then $\lim_{n\rightarrow\infty}\mathcal{D}\left(\boldsymbol{Q}\right)$ is proved to be strictly positive (Lemmas \ref{lemma:Main} and \ref{thm:corl1-2}). Having defined $\mathcal{D}\left(\boldsymbol{Q}\right)$, one needs to move forward and check whether a subset of samples (a row sub-matrix of $\boldsymbol{\mathcal{X}}$) is $\epsilon$-pure or not. The Maximal Total Correlation (MTC) measure defined next is a tool to achieve this goal.  
\\[-3mm]
\begin{definition}[Maximal Total Correlation]
\label{def:Dmax}
The Maximal Total Correlation (MTC) of $\boldsymbol{Y}\in\left\{0,1\right\}^{m\times L}$, a row sub-matrix of $\boldsymbol{\mathcal{X}}$, for a sub-dimension $d\leq L$ is defined as
\begin{equation*}
\mathcal{D}_{\max}\left(\boldsymbol{Y};d\right)
:=
\max_{\boldsymbol{Q}\in \mathrm{Col}\left(\boldsymbol{Y};d\right)}
\mathcal{D}\left(\boldsymbol{Q}\right),
\end{equation*}
where maximization is taken over $\mathrm{Col}\left(\boldsymbol{Y};d\right)$ which consists of all $\binom{L}{d}$ column sub-matrices of $\boldsymbol{Y}$ with size $m\times d$.
\end{definition}

\begin{algorithm}[t]
\caption{:~~BMM clustering via Exhaustive Search}
\label{Algo:Main}
\begin{algorithmic}
\vspace{1mm}
\STATE Inputs: Dataset $\boldsymbol{\mathcal{X}}$, and parameters $\mathcal{L},\delta,\epsilon$ and $\alpha$,
\\[3mm]
\STATE Set $d\leftarrow\frac{1-\alpha}{2\left(\alpha\delta\right)^2\left(1-\epsilon\right)}\left(1+\log\frac{1}{\alpha\epsilon}\right)$
       \hfill (Sub-matrix column size) \hspace*{40mm}
\\[1mm]
\STATE Set $\tau\leftarrow \frac{\epsilon}{2}\left(1+\log\frac{1}{\alpha\epsilon}\right)$ 
       \hfill (Pureness test threshold) \hspace*{40.4mm}
\\
\STATE Set $\kappa\leftarrow 1$ 
       \hfill (Cluster number) \hspace*{53.5mm}
\\
\WHILE{$\kappa <\left\lceil \frac{1}{\alpha}\right\rceil$}
\FOR{$\forall\boldsymbol{Z} \in \left\{1,\ldots,\kappa\right\}^n$, where the size of each cluster is at least $\alpha n/2$}
    \vspace{1mm}
    \STATE Set $\boldsymbol{Y}_1\ldots,\boldsymbol{Y}_{\kappa}\leftarrow$ The clustered row sub-matrices of $\boldsymbol{\mathcal{X}}$ based on $\boldsymbol{Z}$.
    \vspace{1mm}
	\IF{$\mathcal{D}_{\max}\left(\boldsymbol{Y}_k;d\right)\leq \tau$ for $\forall k=1,\ldots,\kappa$}
	\STATE  Set $\boldsymbol{Z}^*\leftarrow\boldsymbol{Z}$, and
	\STATE Terminate the program.
	\ENDIF
\ENDFOR
\STATE Set $\kappa\leftarrow\kappa+1$
\ENDWHILE
\\[3mm]
\STATE Output: $\boldsymbol{Z}^*$, a clustering of data in $\boldsymbol{\mathcal{X}}$.
\end{algorithmic}
\end{algorithm}

The MTC measure is the main tool used in our proposed clustering strategy which is presented in Algorithm \ref{Algo:Main}. In fact, Algorithm \ref{Algo:Main} works by searching over all possible clusterings of the dataset $\boldsymbol{\mathcal{X}}$, which have the following two properties: (i) the number of clusters does not exceed $\left\lceil\frac{1}{\alpha}\right\rceil$, where $\alpha$ is a lower-bound on the probability of the smallest cluster in $\mathcal{B}$; and (ii) the smallest cluster has at least $\alpha n/2$ members. We start with a single cluster and then increase the number of clusters one by one. Given the conditions of Theorem \ref{thm:Main}, the true clustering $\boldsymbol{Z}_T$ would be in this search space with probability at least $1-\zeta/3$.

For any given clustering, we check to see whether all the corresponding clusters are $\epsilon$-pure or not. We can do this by evaluating the MTC over each clustered row sub-matrix of $\boldsymbol{\mathcal{X}}$. If $\mathcal{D}_{\max}$ has negligible values (smaller than a pre-defined threshold $\tau$) in all the clusters, then the current clustering is accepted and the program terminates. We show that under the constraints of Theorem \ref{thm:Main}, the probability of accepting a clustering with even one non $\epsilon$-pure cluster is less than $\zeta/3$. On the other hand, the algorithm eventually reaches the true clustering $\boldsymbol{Z}_T$ by assuming that we have not accepted any other candidates up to that point. Again, we show that the probability of rejecting the true clustering $\boldsymbol{Z}_T$ is no more than $\zeta/3$.

Therefore, one can deduce that with probability at least $1-\zeta$, Algorithm \ref{Algo:Main} either outputs an $\epsilon$-correct clustering on $\boldsymbol{\mathcal{X}}$ or the true clustering $\boldsymbol{Z}_T$ (which of course is $\epsilon$-correct as well). The key property in our analysis that has made it possible for Algorithm \ref{Algo:Main} to work is that our proposed MTC measure can detect the impurity of data clusters with a decision error which decays exponentially w.r.t. $n\times\mathcal{L}$. Even though total correlation has been extensively used in the literature, proving the above-mentioned property and using this measure for clustering discrete mixture model data is novel.


\subsection{Discussions}
Both MTC and its parent, i.e. total correlation, are very powerful in differentiating between pure and non-pure groups of samples. In this section, we elaborate on this fact. Assume a random vector $\boldsymbol{X}\in\left\{0,1\right\}^{L}$ with $\boldsymbol{X}\sim\mathcal{B}\left(K,\boldsymbol{P},\boldsymbol{w}\right)$. If we condition on $\boldsymbol{X}$ to be drawn from a particular mixture component of $\mathcal{B}$, say the $k$th one with $k\in\left\{1,2,\ldots,K\right\}$, then the probability distribution of $\boldsymbol{X}$ would be
\begin{equation}
\mathbb{P}\left(\boldsymbol{X}\vert k\right)=\prod_{\ell=1}^{L}\mathbb{P}\left(X_{\ell}\vert k\right).
\label{eq:singleBM}
\end{equation}
However, according to \eqref{eq:mixture}, the distribution of $\boldsymbol{X}$ without this assumption is $\mathbb{P}\left(\boldsymbol{X}\right)=\sum_{k}w_k\mathbb{P}\left(\boldsymbol{X}\vert k\right)$. A more subtle comparison of $\mathbb{P}\left(\boldsymbol{X}\right)$ and $\mathbb{P}\left(\boldsymbol{X}\vert k\right)$ simply reveals that in a mixture model, unlike the case of a single Bernoulli model, different dimensions of the vector are not necessarily independent from each other. This argument can be qualitatively justified as follows: a group of observed dimensions can convey information about the mixture component to which $\boldsymbol{X}$ belongs, which then impacts the distribution of any other group of dimensions. However, this statistical dependency vanishes when $\boldsymbol{X}$ is known to be generated from a single Bernoulli model.

Based on the above argument, we have used the total correlation measure (defined in Definition \ref{def:DQ}), in order to quantify whether a selected row-subset of samples in $\boldsymbol{\mathcal{X}}$ are more likely to be drawn from a single Bernoulli model, or a mixture of various Bernoulli models with different parameters. When $n$ is finite, our algorithm only considers a relatively small $m\times d$ sub-matrix of $\boldsymbol{\mathcal{X}}$. This is due to the fact that the number of data samples $m$ which are required to make a reliable assessment of the $\epsilon$-pureness of a sub-matrix grows exponentially with respect to $d$ (see Lemmas \ref{lemma:Main} and \ref{lemma:errExponent2} in the next section). This fact should not be surprising since reliable computation of total correlation is subject to having a relatively close estimation of a $d$-dimensional binary distribution. Hence large values of $d$ are unsuitable for estimating $\mathcal{D}\left(\cdot\right)$. On the other hand, by choosing a small $d$, we are ignoring a huge amount of valuable information in the dataset. To exploit all the information embedded in $\boldsymbol{\mathcal{X}}$, the MTC is introduced. It computes numerous total correlations over various subsets of dimensions, and aggregates all these values to form a more informative measure.

\section{Proof of Theorem \ref{thm:Main}}
\label{sec:proof}

We first start by some lemmas which indicate the goodness of our proposed {\it {purity check}} measures. The following lemma shows that $\mathcal{D}\left(\boldsymbol{Q}\right)$ deviates from zero with high probability whenever $\boldsymbol{Q}$ is generated by a BMM with $K\ge 2$ mixture components.
\\[-3mm]
\begin{lemma}
\label{lemma:Main}
Assume ${\mathcal{B}}$ to be a BMM with $K\ge2$ clusters, dimension $d$, frequency matrix $\boldsymbol{P}\in\left[0,1\right]^{K\times d}$ and cluster probability vector $\boldsymbol{w}=\left(w_1,\ldots,w_K\right)$. Let $\boldsymbol{P}$ be $\left(\mathcal{L},\delta\right)$-separable for some $\mathcal{L}\leq d$ and $\delta>0$. Also, assume there exists $\epsilon>0$ such that $w_k\leq 1-\epsilon,~\forall k=1,\ldots,K$. Consider $\boldsymbol{Q}_1,\ldots,\boldsymbol{Q}_n$ to be $n$ i.i.d. samples drawn from ${\mathcal{B}}$, and let $\boldsymbol{Q}=\left[\boldsymbol{Q}_1\vert\ldots\vert\boldsymbol{Q}_n\right]^T\in\left\{0,1\right\}^{n\times d}$. Then, if $\mathcal{L}>\frac{1+\log\frac{K}{\epsilon}}{\left(1-\epsilon\right)\delta^2}$, we have
\begin{equation*}
\mathbb{P}\left\{\mathcal{D}\left(\boldsymbol{Q}\right)\leq \tau\right\}\leq
2^{d+1}e^{-\beta n},
\end{equation*}
where $\tau:= \frac{\epsilon}{2}\left(1+\log\frac{K}{\epsilon}\right)$ and $\beta:= \frac{\tau^2}{d^4 2^{d+1}}$.
\end{lemma}
Proof of Lemma \ref{lemma:Main} is given in Appendix \ref{sec:auxLemmas}. The assumption of $w_k\leq 1-\epsilon$ for all  $k$ yields that when $n\rightarrow\infty$, the set of observations $\boldsymbol{Q}_1,\ldots,\boldsymbol{Q}_n$ would not be $\epsilon$-pure, almost surely. Hence, for an asymptotically large non $\epsilon$-pure set of observations, we have $\mathcal{D}\left(\boldsymbol{Q}\right)\stackrel{a.s.}{\ge}\tau$. On the other hand, we have already discussed that for a completely pure set, i.e. when samples are drawn from a single Bernoulli model ($K=1$), we have $\lim_{n\rightarrow\infty}\mathcal{D}\left(\boldsymbol{Q}\right)\stackrel{a.s.}{=}0$. Accordingly, the following lemma provides a concentration bound on $\mathcal{D}\left(\boldsymbol{Q}\right)$ when rows of $\boldsymbol{Q}$ are drawn from a single Bernoulli model.
\\[-3mm]
\begin{lemma}
\label{lemma:errExponent2}
Let ${\mathcal{B}}$ be a single Bernoulli model ($K=1$) with dimension $d$ and an arbitrary frequency vector. Consider $\boldsymbol{Q}_1,\ldots,\boldsymbol{Q}_n$ to be $n$ i.i.d. samples drawn from ${\mathcal{B}}$,  and let $\boldsymbol{Q}=\left[\boldsymbol{Q}_1\vert\ldots\vert\boldsymbol{Q}_n\right]^T$. Then, we have
\begin{equation*}
\mathbb{P}\left\{
\mathcal{D}\left(\boldsymbol{Q}\right)\ge \tau
\right\}\leq 2^{d+1}e^{-\beta d^2 n},
\end{equation*}
where $\tau$ and $\beta$ are the same as in Lemma \ref{lemma:Main}.
\end{lemma}
The proof for Lemma \ref{lemma:errExponent2} is also given in Appendix \ref{sec:auxLemmas}. And finally, the following lemma shows that the error probability in detecting an improper clustering of samples, i.e. any clustering with at least one non $\epsilon$-pure cluster, drops exponentially with respect to $n\times\mathcal{L}$.
\\[-3mm]
\begin{lemma}
\label{thm:corl1-2}
Assume $\mathcal{B}$ to be a BMM with $K$ clusters,  dimension $L$, frequency matrix $\boldsymbol{P}$ and weight vector $\boldsymbol{w}=\left\{w_1,\ldots,w_K\right\}$. Consider $\boldsymbol{Y}_1,\ldots,\boldsymbol{Y}_n$ to be $n$ i.i.d. samples drawn from ${\mathcal{B}}$ and let $\boldsymbol{Y}=\left[\boldsymbol{Y}_1\vert\ldots\vert\boldsymbol{Y}_n\right]^T\in\left\{0,1\right\}^{n\times L}$. For $K\ge2$, assume $\boldsymbol{P}$ to be $\left(\mathcal{L},\delta\right)$-separable for some $\mathcal{L}\leq L$ and $\delta>0$. Also, assume there exists $\epsilon>0$ such that $w_k\leq 1-\epsilon,~\forall k$. Assume $\mathcal{L}>d:= \frac{K\left(K-1\right)}{2\left(1-\epsilon\right)\delta^2}\left(1+\log\frac{K}{\epsilon}\right)$ and let $\tau:= \frac{\epsilon}{2}\left(1+\log\frac{K}{\epsilon}\right)$. Then, 
\begin{equation*}
\mathbb{P}\left\{
\mathcal{D}_{\max}\left(\boldsymbol{Y};d\right)
\leq
\tau
\right\}
\leq
4^{\mathcal{L}}
\exp\left(
\frac{-\tau^2n\mathcal{L}}{d^52^{d+1}}
\right).
\end{equation*}
On the other hand, when $K=1$ we have:
\begin{equation*}
\mathbb{P}\left\{
\mathcal{D}_{\max}\left(\boldsymbol{Y};d\right)
\ge
\tau
\right\}
\leq
\binom{L}{d}
2^{d+1}
\exp\left(\frac{-\tau^2n}{d^22^{d+1}}\right).
\end{equation*}
\end{lemma}

Based on Lemma \ref{thm:corl1-2}, for sufficiently large $n$ and $\mathcal{L}$, the probability of mis-detection between an $\epsilon$-pure subset of samples in the dataset $\boldsymbol{\mathcal{X}}$ and a non $\epsilon$-pure one is strictly bounded. In fact, Lemma \ref{thm:corl1-2} provides a mathematically rigor and reliable criterion to distinguish between a ``good" and ``bad" clustering of samples in a finite dataset.


The parameter $\alpha$ in Algorithm \ref{Algo:Main} is user-defined. For a BMM $\mathcal{B}=\mathcal{B}\left(K,\boldsymbol{P},\boldsymbol{w}\right)$, as long as we have $\min_{k}w_k\ge\alpha$, $K$ cannot not exceed $\left\lceil\frac{1}{\alpha}\right\rceil$. Given that the conditions in Theorem \ref{thm:Main} are satisfied, with probability at least $1-\zeta$ Algorithm \ref{Algo:Main} terminates before passing $\left\lceil\frac{1}{\alpha}\right\rceil$ clusters and as soon as it finds an $\epsilon$-correct clustering of $\boldsymbol{\mathcal{X}}$. Otherwise, the algorithm just outputs a null clustering. In the following, we use the results from Lemma \ref{thm:corl1-2} to prove Theorem \ref{thm:Main}, which is also the mathematical analysis of Algorithm \ref{Algo:Main}.

\begin{proof}[Proof of Theorem \ref{thm:Main}]
Algorithm \ref{Algo:Main} checks all possible cluster numbers $\kappa\leq\left\lceil\frac{1}{\alpha}\right\rceil$, starting from $\kappa=1$. Let us denote the number of clusterings that need to be checked before reaching the true latent clustering $\boldsymbol{Z}_T$ by $N$. Then, $N$ obviously satisfies the following inequality:
\begin{equation*}
N\leq 1^n +2^n + \ldots + K^n\leq K^{n+1}.
\end{equation*}
In this regard, one can consider the following error events during the execution of Algorithm \ref{Algo:Main}:
\begin{itemize}
\item 
$\mathcal{E}_1$: Accepting a non $\epsilon$-correct clustering of dataset $\boldsymbol{\mathcal{X}}$, before reaching the true clustering $\boldsymbol{Z}_T$. Recall that a non $\epsilon$-correct clustering denotes any clustering with at least one non $\epsilon$-pure cluster.
\item
$\mathcal{E}_2$: Eventually reaching to the true clustering $\boldsymbol{Z}_T$, and denying it.
\item
$\mathcal{E}_3$: The smallest true cluster in the dataset $\boldsymbol{\mathcal{X}}$ has less than $\alpha n/2$ members.
\end{itemize}
Obviously, probability of the algorithm failure, denoted by $P_E$, can be upper-bounded as
\begin{equation*}
P_E=
\mathbb{P}\left\{\mathcal{E}_1\cup\mathcal{E}_2\cup\mathcal{E}_3\right\}\leq
\mathbb{P}\left\{\mathcal{E}_1\right\}+\mathbb{P}\left\{\mathcal{E}_2\right\}+\mathbb{P}\left\{\mathcal{E}_3\right\}.
\end{equation*}
In the following, we show that given the conditions of Theorem \ref{thm:Main}, we have $\mathbb{P}\left\{\mathcal{E}_i\right\}\leq\zeta/3$ for $i=1,2,3$.

From Lemma \ref{thm:corl1-2}, we know that the probabilities of accepting a non $\epsilon$-pure clustering, and rejecting the correct one are both bounded and decrease exponentially w.r.t. $n$. In the following, we compute the corresponding error exponents for the particular parameter setting of Algorithm \ref{Algo:Main}. Recall sub-dimension $d$ and threshold $\tau$ as
\begin{equation}
d:=
\frac{1-\alpha}{2\left(\alpha\delta\right)^2\left(1-\epsilon\right)}
\left(1+\log\frac{1}{\alpha\epsilon}\right)=O\left(\log\frac{1}{\epsilon}\right),
~~~\mathrm{and}~~~
\tau:=
\frac{\epsilon}{2}\left(1+\log\frac{1}{\alpha\epsilon}\right)=
O\left(\epsilon\log\frac{1}{\epsilon}\right).
\label{eq:order:tau-d}
\end{equation}
As it becomes evident in the proceeding parts of the proof, we also need to compute the order of $e^d$ w.r.t. $\epsilon$. In this regard, one can write:
\begin{equation*}
e^d = \left(e/\alpha\right)^{\frac{1-\alpha}{2\left(\alpha\delta\right)^2\left(1-\epsilon\right)}}
\cdot
\left(1/\epsilon\right)^{\frac{1-\alpha}{2\alpha^2\delta^2\left(1-\epsilon\right)}}.
\end{equation*}
The first term is $O\left(1\right)$ with respect to $\epsilon$, when $\epsilon\rightarrow0$. Therefore, we have
\begin{equation}
e^d = O\left(
\left(1/\epsilon\right)^{\frac{1-\alpha}{2\left(\alpha\delta\right)^2}}
\right).
\label{eq:order:ed}
\end{equation}
By using the union bound over all non $\epsilon$-pure clusterings in the first $N$ steps of the algorithm, we have the following inequality:
\begin{align}
\mathbb{P}\left\{\mathcal{E}_1\right\}\leq N\mathbb{P}\left\{\mathcal{E}^{\left(1\right)}_1\right\},
\label{eq:algoBound1prim}
\end{align}
where $\mathcal{E}^{\left(1\right)}_1$ represents the error event corresponding to the acceptance of a  single non $\epsilon$-pure clustering. In \eqref{eq:algoBound1prim}, the factor $N$ is an upper bound on the number of non $\epsilon$-correct clusterings that need to be checked by Algorithm \ref{Algo:Main} before reaching the true clustering $\boldsymbol{Z}_T$. Moreover, it should be noted that for a non $\epsilon$-correct clustering $\boldsymbol{Z}\in\left\{1,\ldots,\kappa\right\}^n$, at least one of the clusters is not $\epsilon$-pure, and thus we need our MTC measure to detect at least one of such clusters. 

Remember that all clusters are assumed to have at least $\alpha n/2$ samples, and also $N\leq K^{n+1}$. This way, by using Lemma \ref{thm:corl1-2} the following bound on the probability of error event $\mathcal{E}_1$ can be attained:
\begin{align}
\mathbb{P}\left\{\mathcal{E}_1\right\}
&\leq
K^{n+1}\cdot4^{\mathcal{L}}\cdot
\exp\left(\frac{-\alpha{\tau}^2 n\mathcal{L}}{d^5 2^{d+2}}\right)
\nonumber\\
&=\exp\left(
O\left(n\right)+O\left(\mathcal{L}\right)-
n\mathcal{L}
\cdot
\frac{\alpha\tau^2}{d^5 2^{d+2}}
\right)
\nonumber\\
&\leq\exp\left(
O\left(n\right)+O\left(\mathcal{L}\right)-
n\mathcal{L}
\cdot
\frac{\alpha\tau^2}{d^5 e^{d+2}}
\right)
\nonumber\\
&=\exp\left(
O\left(n\right)+O\left(\mathcal{L}\right)-
n\mathcal{L}
\cdot
O\left(
\frac{\epsilon^{2+
\frac{1-\alpha}{2\left(\alpha\delta\right)^2}}}
{\log^{3}\left({1}/{\epsilon}\right)}
\right)
\right),
\label{eq:algoBoundMain}
\end{align}
where we have used the results of equations 
\eqref{eq:order:tau-d} and \eqref{eq:order:ed}. The dominant exponent in the r.h.s. of \eqref{eq:algoBoundMain} corresponds to the $n\times\mathcal{L}$ term. In other words, by choosing sufficiently large $n$ and $\mathcal{L}$, one can make $\mathbb{P}\left\{\mathcal{E}_1\right\}$ arbitrarily small, even though the union bound is over $K^{n+1}$ events. Mathematically speaking, it can be confirmed that by choosing
\begin{equation}
\mathcal{L} \ge \frac{B\log^{3}\left(1/\epsilon\right)}
{\epsilon^{2+\frac{1-\alpha}{2\left(\alpha\delta\right)^2}}}
\quad,\quad
n\ge n^{\left(1\right)}_{\min}:= C^{\left(1\right)}\left(
\frac{\log^{3}\left(1/\epsilon\right)}
{\epsilon^{2+\frac{1-\alpha}{2\left(\alpha\delta\right)^2}}}
+\log\frac{1}{\zeta}
\right),
\label{eq:firstLbound}
\end{equation}
we can achieve $\mathbb{P}\left\{\mathcal{E}_1\right\}\leq\zeta/3$ for arbitrary small $\epsilon,\zeta>0$, where coefficients $B$ and $C^{\left(1\right)}$ do not depend on $\epsilon$ or $\zeta$.

A similar argument can be used to  obtain an upper-bound on $\mathbb{P}\left\{\mathcal{E}_2\right\}$. It should be noted that for $\mathcal{E}_2$ to occur, at least one of the true clusters in $\boldsymbol{Z}_T$ must have $\mathcal{D}_{\max}>\tau$. Since the number of clusters at that step of the algorithm is $K$, one can use the union bound over all $K$ clusters each of which has at least $\alpha n/2$ members. Also, we aim to use the following inequality:
\begin{equation*}
\log\binom{L}{d}\leq d\log\frac{Le}{d}.
\end{equation*}
In this regard, by using the second inequality in Lemma \ref{thm:corl1-2}, it can be shown that
\begin{align}
\mathbb{P}\left\{\mathcal{E}_2\right\}\leq
K\binom{L}{d}2^{d+1}\exp\left(
\frac{-\alpha \tau^2n}{d^2 2^{d+2}}
\right)
=\exp\left(O\left(\log\frac{1}{\epsilon}\left[1+\log\frac{L}{\log\left(1/\epsilon\right)}\right]\right)
-nO\left(\epsilon^{2+\frac{1-\alpha}{2\left(\alpha\delta\right)^2}}\right)
\right).
\label{eq:algoBound2}
\end{align}
Thus, by choosing
\begin{equation}
n\ge n^{\left(2\right)}_{\min}:=
C^{\left(2\right)}\left(
\frac{\log\left(1/\epsilon\right)\log L}
{\epsilon^{2+\frac{1-\alpha}{2\left(\alpha\delta\right)^2}}}
+\log\frac{1}{\zeta}
\right),
\end{equation}
we have $\mathbb{P}\left\{\mathcal{E}_2\right\}\leq\zeta/3$, where $C^{\left(2\right)}$ is a constant that does not depend on $\epsilon$ or $\zeta$. Finally, according to Lemma \ref{lemma:minClusterMember}, by choosing $n\ge n^{\left(3\right)}_{\min}$ which is defined as
\begin{equation*}
n^{\left(3\right)}_{\min}
:= C^{\left(3\right)}\log\frac{1}{\zeta},
\end{equation*}
one can guarantee that the probability of occurring $\mathcal{E}_3$ is less that $\zeta/3$, where again constant $C^{\left(3\right)}$ is independent of $\epsilon$ or $\zeta$. Therefore, assuming that $\mathcal{L}$ satisfies the inequality in \eqref{eq:firstLbound} and the sample size $n$ satisfies
\begin{equation*}
n\ge
\frac{C\log^3\left(1/\epsilon\right)}{\epsilon^{2+\frac{1-\alpha}{2\left(\alpha\delta\right)^2}}}
\log\frac{L}{\zeta}
\ge \max\left\{n^{\left(1\right)}_{\min},n^{\left(2\right)}_{\min},n^{\left(3\right)}_{\min}\right\}
\end{equation*}
for some constant $C$, Algorithm \ref{Algo:Main} is guaranteed to output an $\epsilon$-correct clustering on dataset $\boldsymbol{\mathcal{X}}$ with probability at least $1-\zeta$. This completes the proof.
\end{proof}

Let us consider an asymptotic regime where dimension $L$ is being increased while the number of informative dimensions $\mathcal{L}$ is kept fixed. Then, according to Theorem \ref{thm:Main}, the minimum required dataset size $n$ should grow logarithmically w.r.t. $L$. This analytic observation makes sense since in finite $n$ regimes, addition of more {\it {non-informative}} dimensions, i.e. those dimensions that have the same frequency values between all or at least some of the clusters, only adds extra noise to the dataset $\boldsymbol{\mathcal{X}}$ and thus makes the clustering a more challenging task.

\section{Conclusions}
\label{sec:conc}

This paper aims to find the first sample complexity bounds on reliable clustering of Bernoulli Mixture Models, when the number of clusters is unknown. To this aim, we propose a novel variant of an existing measure in statistics, denoted by Maximal Total Correlation (MTC), and show it has interesting concentration properties. Based on this measure, we propose an algorithm that is capable of clustering the data with a maximum mis-clustering rate of $\epsilon$ with probability at least $1-\zeta$ (for any $\epsilon,\zeta>0$), as long as sample complexity $n$ and the number of informative dimensions $\mathcal{L}$ grow polynomially w.r.t. $\epsilon$ and $\zeta$. No restrictive assumptions have been made in our model, except those that are required for the meaningfulness of clustering, such as: existence of a non-zero difference among frequency vectors of different mixture components, and a minimum weight for each cluster in the model. As a result, our findings encapsulate many classes of BMM inference problems.

Our focus in this paper is not on the computational efficiency. As a result, the proposed Algorithm \ref{Algo:Main} is exponential-time w.r.t. $n$, which means efficient PAC-learnability of BMMs still remains an an open problem. In general, the existence of an {\it {efficient}} (polynomial-time) algorithm for density estimation, clustering or parameter identification of many mixture models is unknown \cite{ashtiani2018nearly}, including Gaussian mixture models. Therefore, obtaining an efficient method for clustering of BMMs is both theoretically and practically important. On the other hand, we are not aware of any sample complexity lower bounds for clustering or learning of BMMs. Therefore, it is not clear whether the upper bounds in this paper are tight or not. Deriving lower bounds for Theorem \ref{thm:Main} is also a good direction for future works in this area.

\bibliographystyle{IEEEtran}
\bibliography{IEEEabrv,ref}

\appendix
\numberwithin{equation}{section}
\numberwithin{thm}{section}
\numberwithin{lemma}{section}
\numberwithin{definition}{section}
\numberwithin{corl}{section}
\numberwithin{figure}{section}

\section{Auxiliary Lemmas and Proofs}
\label{sec:auxLemmas}

\begin{proof}[Proof of Lemma \ref{lemma:Main}]
Proof consists of two parts.
First, we show $\mathcal{D}\left(\boldsymbol{Q}\right)$ is almost surely greater than the positive threshold $2\tau$ in the asymptotic case, i.e.
\begin{equation}
\mathbb{P}\left\{\lim_{n\rightarrow\infty}\mathcal{D}\left(\boldsymbol{Q}\right) \leq 2\tau\right\}=0.
\end{equation} 
Second, we prove that the probability of $\left\vert \mathcal{D}\left(\boldsymbol{Q}\right) - \lim_{n\rightarrow\infty}\mathcal{D}\left(\boldsymbol{Q}\right)\right\vert>\tau$ decays exponentially w.r.t. $n$, which complete the proof. 

For the sake of simplicity in this proof, let $\mathbb{P}_{{\boldsymbol{p}}}$ represent the probability distribution of a Bernoulli model with frequency vector ${\boldsymbol{p}}\in\left[0,1\right]^{d}$. This way, one can write
\begin{equation}
\lim_{n\rightarrow\infty}\mathcal{D}\left(\boldsymbol{Q}\right)
\stackrel{a.s.}{=}
\mathcal{D}_{\mathrm{KL}}\left(
\sum_{k}w_k\mathbb{P}_{{\boldsymbol{p}}^{\left(k\right)}}
\bigg\Vert
\mathbb{P}_{\Bar{\boldsymbol{p}}}
\right),
\label{eq:limitDMeasure}
\end{equation}
where ${\boldsymbol{p}}^{\left(k\right)}$ denotes the $k$th row of frequency matrix $\boldsymbol{P}$, and $\Bar{\boldsymbol{p}}:= \sum_{k}w_k{\boldsymbol{p}}^{\left(k\right)}$. Here, $\sum_{k}w_k\mathbb{P}_{{\boldsymbol{p}}^{\left(k\right)}}$ indicates a mixture of Bernoulli models (a BMM), while $\mathbb{P}_{\Bar{\boldsymbol{p}}}$ denotes a single Bernoulli model whose frequency vector is the weighted average of the $K$ frequency vectors in $\boldsymbol{P}$.

It can be verified that when only one component of $\boldsymbol{w}$ is $1$ and the rest are $0$, the two probability distributions $\sum_{k}w_k\mathbb{P}_{{\boldsymbol{p}}^{\left(k\right)}}$ and $\mathbb{P}_{\Bar{\boldsymbol{p}}}$ are equal and $\lim_{n\rightarrow\infty}\mathcal{D}\left(\boldsymbol{Q}\right)\stackrel{a.s.}{=}0$. However, if for some $\epsilon>0$ we have $w_k\leq 1-\epsilon,~\forall k$, and $\mathcal{L}$ is {\it {sufficiently}} large, then we prove that $\sum_{k}w_k\mathbb{P}_{{\boldsymbol{p}}^{\left(k\right)}}$ cannot be consistently approximated by a single Bernoulli model with frequency vector $\sum_{k}w_k{\boldsymbol{p}}^{\left(k\right)}$.

Based on the definition of the Kullback-Liebler divergence, r.h.s. of \eqref{eq:limitDMeasure} can be expanded as follows which helps us to find a proper lower-bound for $\lim_{n\rightarrow\infty}\mathcal{D}\left(\boldsymbol{Q}\right)$:
\begin{align}
\mathcal{D}_{\mathrm{KL}}\left(
\sum_{k}w_k\mathbb{P}_{{\boldsymbol{p}}^{\left(k\right)}}
\bigg\Vert
\mathbb{P}_{\Bar{\boldsymbol{p}}}
\right)
&=
\sum_{\boldsymbol{X}\in\left\{0,1\right\}^d}\sum_{k}w_k\mathbb{P}_{{\boldsymbol{p}}^{\left(k\right)}}\left(\boldsymbol{X}\right)
\log\left(
\frac{\sum_{u}w_u\mathbb{P}_{{\boldsymbol{p}}^{\left(u\right)}}\left(\boldsymbol{X}\right)}
{\mathbb{P}_{\Bar{\boldsymbol{p}}}\left(\boldsymbol{X}\right)}
\right)
\nonumber\\
&
\ge
\sum_{\boldsymbol{X}\in\left\{0,1\right\}^d}\sum_{k}w_k\mathbb{P}_{{\boldsymbol{p}}^{\left(k\right)}}\left(\boldsymbol{X}\right)
\log\left(
\frac{w_k\mathbb{P}_{{\boldsymbol{p}}^{\left(k\right)}}\left(\boldsymbol{X}\right)}
{\mathbb{P}_{\Bar{\boldsymbol{p}}\left(\boldsymbol{X}\right)}}
\right)
\nonumber\\
&
=
\sum_{k}w_k\sum_{\ell=1}^{d}\left(
\sum_{X_{\ell}\in\left\{0,1\right\}}
\mathbb{P}_{p^{\left(k\right)}_\ell}\left(X_{\ell}\right)
\log\left(
\frac{\mathbb{P}_{p^{\left(k\right)}_\ell}\left(X_{\ell}\right)}
{\mathbb{P}_{\Bar{p}_{\ell}}\left(X_{\ell}\right)}
\right)
\right)
-\mathbb{H}\left(\boldsymbol{w}\right)
\nonumber\\
&
=
\sum_{\ell=1}^{d}\sum_{k=1}^{K}
w_k\mathcal{D}_{\mathrm{KL}}\left(
\mathbb{P}_{p^{\left(k\right)}_\ell}
\big\Vert
\mathbb{P}_{\Bar{p}_\ell}
\right)
-\mathbb{H}\left(\boldsymbol{w}\right),
\label{eq:lemma1MainLB}
\end{align}
where $\mathbb{H}\left(\boldsymbol{w}\right):= -\sum_{k}w_k\log w_k$ denotes the Shannon entropy of the discrete distribution $\boldsymbol{w}$. Moreover, it is easy show that
\begin{equation}
\sum_{k=1}^{K}w_k\mathcal{D}_{\mathrm{KL}}\left(
\mathbb{P}_{p^{\left(k\right)}_\ell}
\big\Vert
\mathbb{P}_{\Bar{p}_\ell}
\right)
=
H\left(\sum_{k=1}^{K}w_kp^{\left(k\right)}_{\ell}\right)
-\sum_{k=1}^{K}w_k H\left(p^{\left(k\right)}_{\ell}\right),
\label{eq:RevAskedToNumber}
\end{equation}
where, for the simplicity of notation, $H\left(p\right)$ for $0\leq p\leq1$ refers to $H\left(p\right):= \mathbb{H}\left(\mathrm{Bern}\left(p\right)\right)=-p\log p - \left(1-p\right)\log\left(1-p\right)$. Since $H\left(\cdot\right)$ is a strictly concave function, and considering the fact that $\mathbb{H}\left(\boldsymbol{w}\right)$ is always upper-bounded by $\log K$ regardless of $\mathcal{L}$, one can conclude that the lower-bound for $\lim_{n\rightarrow\infty}\mathcal{D}\left(\boldsymbol{Q}\right)$ in \eqref{eq:lemma1MainLB} becomes strictly positive when i) $\mathcal{L}$ is sufficiently large, and ii) frequency vectors $\boldsymbol{p}^{\left(k\right)}$ for $k=1,\ldots,K$ are sufficiently far apart from each other. 

In order to simplify the lower-bound in \eqref{eq:lemma1MainLB}, let us assume a random variable $\boldsymbol{A}\in\left[0,1\right]$, and define $\boldsymbol{a}:=\boldsymbol{A}-\mathbb{E}\boldsymbol{A}$. According to Taylor's theorem \cite{courant2011differential}, one can write
\begin{align}
H\left(\mathbb{E}\boldsymbol{A}\right)-\mathbb{E}H\left(\boldsymbol{A}\right)&=
H\left(\mathbb{E}\boldsymbol{A}\right)-\mathbb{E}\left\{
H\left(\mathbb{E}\boldsymbol{A}\right) + H'\left(\mathbb{E}\boldsymbol{A}\right)\boldsymbol{a}+
\frac{1}{2}H''\left(\mathbb{E}\boldsymbol{A}+\xi\right)\boldsymbol{a}^2
\right\}
\nonumber\\
&\ge
\frac{\mathbb{E}\boldsymbol{a}^2}{2}\inf_{0\leq p\leq 1}\left\vert H''\left(p\right)\right\vert
=
\inf_{0\leq p\leq1}
\frac{\mathbb{E}\boldsymbol{a}^2}{2p\left(1-p\right)}
= 2\mathrm{var}\left(\boldsymbol{A}\right),
\label{eq:ARVineq}
\end{align}
where $\xi$ is a random variable depending on $\boldsymbol{A}$, and $\mathrm{var}\left(\boldsymbol{A}\right)=\mathbb{E}\boldsymbol{a}^2$ denotes the variance of $\boldsymbol{A}$. Now, for $\ell=1,\ldots,d$, let us define $\boldsymbol{A}_{\ell}$ as a random variable that takes the values $p^{\left(1\right)}_{\ell},\ldots,p^{\left(K\right)}_{\ell}$ with probabilities $w_1,\ldots,w_K$, respectively. Using the inequality in \eqref{eq:ARVineq}, the lower-bound for $\lim_{n\rightarrow\infty}\mathcal{D}\left(\boldsymbol{Q}\right)$ can be written as
\begin{align}
\lim_{n\rightarrow\infty}\mathcal{D}\left(\boldsymbol{Q}\right)
&\stackrel{a.s.}{\ge}
\sum_{\ell=1}^{d}\left[H\left(\mathbb{E}\boldsymbol{A}_{\ell}\right)
-\mathbb{E}H\left(\boldsymbol{A}_{\ell}\right)\right]-
\mathbb{H}\left(\boldsymbol{w}\right)
\ge
2\sum_{\ell=1}^{d}\mathrm{var}\left(\boldsymbol{A}_{\ell}\right)-
\mathbb{H}\left(\boldsymbol{w}\right)
\nonumber\\
&=
2\sum_{\ell=1}^{d}
\sum_{k=1}^{K}w_k\left(p_{\ell}^{\left(k\right)}-\sum_{u=1}^{K}w_u p_{\ell}^{\left(u\right)}\right)^2-
\mathbb{H}\left(\boldsymbol{w}\right).
\label{eq:semiMainBound}
\end{align}
We have already assumed that $w_{k}\leq 1-\epsilon, \forall k$. Also, due to the $\left(\mathcal{L},\delta\right)$-separability assumption, for all pairs of rows in $\boldsymbol{P}$, say $i$ and $j$, there exists a subset of columns $\mathscr{C}_{i,j}\subseteq\left\{1,2,\ldots,d\right\}$ where
\begin{equation*}
\left\vert p^{\left(i\right)}_\ell - p^{\left(j\right)}_\ell \right\vert \ge \delta~,~\ell\in\mathscr{C}_{i,j},
\end{equation*}
and $\left\vert \mathscr{C}_{i,j}\right\vert \ge \mathcal{L}$. This suggests that the values of $\mathrm{var}\left(\boldsymbol{A}_{\ell}\right)$, at least for $\ell\in\cup_{i,j}\mathcal{C}_{i,j}$, are greater than or equal to a positive function of $\epsilon$ and $\delta$. On the other hand, the only negative term $-\mathbb{H}\left(\boldsymbol{w}\right)$ is bounded and does not scale with $\mathcal{L}$. This suggests that for a large enough $\mathcal{L}$, the r.h.s. of \eqref{eq:semiMainBound} becomes strictly positive.

Lemma \ref{lemma:semiMainBound} proves that the lower-bound in \eqref{eq:semiMainBound}
can be further simplified as
\begin{align*}
\lim_{n\rightarrow\infty}\mathcal{D}\left(\boldsymbol{Q}\right)
~\stackrel{a.s.}{\ge}~
2\mathcal{L}\epsilon\left(1-\epsilon\right)\delta^2-
\mathbb{H}\left(\boldsymbol{w}^*\right)
~\ge~
2\epsilon\left(1-\epsilon\right)\delta^2
\left(\mathcal{L}-\frac{1+\log\frac{K}{\epsilon}}{2\left(1-\epsilon\right)\delta^2}\right),
\end{align*}
where $\boldsymbol{w}^*$ denotes a a weight vector, or equivalently a discrete probability distribution supported on $\left\{1,\ldots,K\right\}$, with $w^*_1=1-\epsilon$ and $w^*_i=\epsilon/\left(K-1\right)$ for $i=2,\ldots,K$. Note that the second inequality directly results from
\begin{align*}
\mathbb{H}\left(\boldsymbol{w}^*\right)=
\left(1-\epsilon\right)\log\frac{1}{1-\epsilon}+\epsilon\log\frac{K-1}{\epsilon}
\leq
\epsilon\left(1+\log\frac{K}{\epsilon}\right).
\end{align*}
Also, we have already assumed that $\mathcal{L}\ge\frac{1+\log\frac{K}{\epsilon}}{\left(1-\epsilon\right)\delta^2}$, which means the following relations hold:
\begin{align*}
\lim_{n\rightarrow\infty}\mathcal{D}\left(\boldsymbol{Q}\right)&
\stackrel{a.s.}{\ge}
2\epsilon\left(1-\epsilon\right)\delta^2
\left(\mathcal{L}-\frac{1+\log\frac{K}{\epsilon}}{2\left(1-\epsilon\right)\delta^2}\right)
\\
&\ge
2\epsilon\left(1-\epsilon\right)\delta^2
\left(
\frac{1+\log\frac{K}{\epsilon}}{\left(1-\epsilon\right)\delta^2}
-\frac{1+\log\frac{K}{\epsilon}}{2\left(1-\epsilon\right)\delta^2}\right)
\\
&\ge
\epsilon\left(1+\log\frac{K}{\epsilon}\right) = 2\tau,
\end{align*}
where the last equality is due to the definition of $\tau:=\frac{\epsilon}{2}\left(1+\log\frac{K}{\epsilon}\right)$ in the statement of Lemma \ref{lemma:Main}. This way, the first part of the proof is complete.

So far, we have shown that $\mathcal{D}\left(\boldsymbol{Q}\right)$ almost surely becomes greater than $2\tau$ when $n$ goes to infinity.
However, $\mathcal{D}\left(\boldsymbol{Q}\right)$ is supposed to be computed over a finite sample size of $n$, thus it is necessary to show that $\left\vert \mathcal{D}\left(\boldsymbol{Q}\right) - \lim_{n\rightarrow\infty}\mathcal{D}\left(\boldsymbol{Q}\right)\right\vert$ concentrates around zero w.r.t. $n$. In fact, Lemma \ref{lemma:errExponent} proves that the probability of the above error term exceeding $\tau$ decays exponentially with respect to $n$. Based on the result of Lemma \ref{lemma:errExponent}, the probability $\mathbb{P}\left\{\mathcal{D}\left(\boldsymbol{Q}\right)\leq \tau\right\}$ can be upper-bounded as
\begin{align*}
\mathbb{P}\left\{\mathcal{D}\left(\boldsymbol{Q}\right)\leq \tau\right\}\leq
2^{d+1}\exp\left(\frac{-n\epsilon^2\left(1-\epsilon\right)^2\delta^4}{d^4 2^{d+1}}\left(\mathcal{L}-\frac{1+\log\frac{K}{\epsilon}}{2\left(1-\epsilon\right)\delta^2}\right)^2\right)
\leq
2^{d+1}\exp\left(\frac{-n\tau^2}{d^4 2^{d+1}}\right),
\end{align*}
which completes the proof.
\end{proof}

\begin{lemma}
\label{lemma:semiMainBound}
The lower-bound for $\lim_{n\rightarrow\infty}\mathcal{D}\left(\boldsymbol{Q}\right)$ in \eqref{eq:semiMainBound}, subject to
$\left(\mathcal{L},\delta\right)$-separability of the frequency matrix $\boldsymbol{P}$ and assuming
$w_k\leq 1-\epsilon,\forall k$, is as follows:
\begin{equation*}
\lim_{n\rightarrow\infty}\mathcal{D}\left(\boldsymbol{Q}\right)
\stackrel{a.s.}{\ge}
2\epsilon\left(1-\epsilon\right)\delta^2\left(\mathcal{L}
-\frac{1+\log\frac{K}{\epsilon}}{2\left(1-\epsilon\right)\delta^2}\right).
\end{equation*}
\end{lemma}
\begin{proof}
Considering the assumption made in Lemma \ref{lemma:Main} with respect to the non $\epsilon$-purity of $\boldsymbol{Q}$, let us define $\boldsymbol{W}\left(r\right)$ for $\epsilon\leq r\leq 1-1/K$ as
\begin{equation*}
\boldsymbol{W}\left(r\right):= \left\{
\boldsymbol{w}\in\mathbb{R}^K~\bigg\vert~
w_k\ge0,~
\sum_{k}w_k=1,~
\max_k~w_k=1-r
\right\}.
\end{equation*}
Hence, according to \eqref{eq:semiMainBound} the lower-bound for $\lim_{n\rightarrow\infty}\mathcal{D}\left(\boldsymbol{Q}\right)$ (for a non $\epsilon$-pure $\boldsymbol{Q}$) can be written as
\begin{equation}
\lim_{n\rightarrow\infty}\mathcal{D}\left(\boldsymbol{Q}\right)
\stackrel{a.s.}{\ge}
\inf_{r\in\left[\epsilon,1-{1}/{K}\right]}
\left\{
\inf_{\boldsymbol{w}\in\boldsymbol{W}\left(r\right)}
2\sum_{\ell=1}^{d}
\sum_{k=1}^{K}w_k\left(p_{\ell}^{\left(k\right)}-\Bar{p}_{\ell}\left(\boldsymbol{w}\right)\right)^2-
\mathbb{H}\left(\boldsymbol{w}\right)
\right\},
\label{eq:lemmaEq2Inf}
\end{equation}
with $\Bar{p}_{\ell}\left(\boldsymbol{w}\right):=\sum_{k=1}^{K}w_k p_{\ell}^{\left(k\right)}$. In fact, \eqref{eq:lemmaEq2Inf} indicates minimization of the lower-bound over all asymptotically large non $\epsilon$-pure matrices $\boldsymbol{Q}$. In order to further simplify the above lower-bound, minimization over $\boldsymbol{w}\in\boldsymbol{W}\left(r\right)$ can be carried out for the two terms in the r.h.s. of \eqref{eq:lemmaEq2Inf}, in a separate manner. Mathematically speaking,
\begin{align*}
\inf_{\boldsymbol{w}\in\boldsymbol{W}\left(r\right)}
\left\{
2\sum_{\ell=1}^{d}
\sum_{k=1}^{K}w_k\left(p_{\ell}^{\left(k\right)}-\Bar{p}_{\ell}\left(\boldsymbol{w}\right)\right)^2-
\mathbb{H}\left(\boldsymbol{w}\right)
\right\}&
\nonumber\\
\ge
\inf_{\boldsymbol{w}\in\boldsymbol{W}\left(r\right)}
&2\sum_{\ell=1}^{d}
\sum_{k=1}^{K}w_k\left(p_{\ell}^{\left(k\right)}-\Bar{p}_{\ell}\left(\boldsymbol{w}\right)\right)^2-
\sup_{\boldsymbol{w}\in\boldsymbol{W}\left(r\right)}
\mathbb{H}\left(\boldsymbol{w}\right).
\end{align*}
It should be reminded that for each weight vector (or equivalently, probability distributions) $\boldsymbol{w}\in\boldsymbol{W}\left(r\right)$, one of the components is exactly equal to $1-r$, and thus the rest of the components must sum to $r$. Therefore, the maximum Shannon entropy $\mathbb{H}\left(\boldsymbol{w}\right)$ occurs when the latter $K-1$ components have an equal probability, i.e. $r/\left(K-1\right)$, which indicates maximum possible randomness. In this regard, it is easy to see that
\begin{equation}
\sup_{\boldsymbol{w}\in\boldsymbol{W}\left(r\right)}\mathbb{H}\left(\boldsymbol{w}\right)
=
\left(1-r\right)\log\frac{1}{1-r}+\sum_{k=2}^{K}\frac{r}{K-1}\log\frac{K-1}{r}
~\leq~
r\left(1+\log\frac{K}{r}\right),
\label{eq:lemmaEq2}
\end{equation}
which is also based on the fact that $\left(1-r\right)\log\frac{1}{1-r}\leq r$. On the other hand, for the first term in r.h.s of \eqref{eq:lemmaEq2Inf}, the following lower-bound can be obtained:
\begin{align*}
\inf_{\boldsymbol{w}\in\boldsymbol{W}\left(r\right)}
2\sum_{\ell=1}^{d}
\sum_{k=1}^{K}
w_k\left(p_{\ell}^{\left(k\right)}-\Bar{p}_{\ell}\left(\boldsymbol{w}\right)\right)^2
&\ge
\inf_{\boldsymbol{w}\in\boldsymbol{W}\left(r\right)}
2\sum_{\ell=1}^{d}
\inf_{p\in\mathbb{R}}\left\{
\sum_{k=1}^{K}w_k\left(p_{\ell}^{\left(k\right)}-p\right)^2
\right\}
\\
&\ge
2\sum_{\ell=1}^{d}\inf_{p\in\mathbb{R}}~
\min_{t=1,\ldots,K}
\left\{
\left(1-r\right)\left(p^{\left(t\right)}_{\ell} - p\right)^2 + \min_{k\vert k\neq t}r
\left(p^{\left(k\right)}_{\ell} - p\right)^2
\right\}
\\
&= 2\sum_{\ell=1}^{d}\min_{k,t=1,\ldots,K\vert~k\neq t}r\left(1-r\right)
\left(p^{\left(k\right)}_{\ell}-p^{\left(t\right)}_{\ell}\right)^2.
\end{align*}
The last equality can be achieved by solving for $\inf_{p\in\mathbb{R}}$, analytically. Since for every pair $\left(k,t\right),~k\neq t$ and at least $\mathcal{L}$ dimensions out of $\ell=1,2,\ldots,d$, the inequality $\left\vert p^{\left(k\right)}_{\ell}-p^{\left(t\right)}_{\ell}\right\vert\ge \delta$ holds, one can write
\begin{equation}
\inf_{\boldsymbol{w}\in\boldsymbol{W}\left(r\right)}
2\sum_{\ell=1}^{d}
\sum_{k=1}^{K}w_k\left(p_{\ell}^{\left(k\right)}-\Bar{p}_{\ell}\left(\boldsymbol{w}\right)\right)^2
\ge
2r\left(1-r\right)\mathcal{L}\delta^2.
\label{eq:lemmaEq1}
\end{equation}
By combining the inequalities in \eqref{eq:lemmaEq2} and \eqref{eq:lemmaEq1}, the following lower-bound can be achieved for $\lim_{n\rightarrow\infty}\mathcal{D}\left(\boldsymbol{Q}\right)$:
\begin{equation}
\lim_{n\rightarrow\infty}\mathcal{D}\left(\boldsymbol{Q}\right)
\stackrel{a.s.}{\ge}
\inf_{r\in\left[\epsilon,1-{1}/{K}\right]}~
2r\left(1-r\right)\delta^2\left(\mathcal{L} - \frac{1+\log\frac{K}{r}}{2\left(1-r\right)\delta^2}
\right).
\label{eq:finLBderiv}
\end{equation}
The objective function in the r.h.s. of \eqref{eq:finLBderiv} is a monotonically increasing function w.r.t. $r$, when $\mathcal{L}>\frac{1+\log\frac{K}{\epsilon}}{2\left(1-\epsilon\right)\delta^2}$. This can be easily verified by taking derivatives w.r.t. $r$ for a sufficiently small $\epsilon$. Hence, the minimizer of $\inf_{r\in\left[\epsilon,1-{1}/{K}\right]}$ occurs when $r=\epsilon$, which completes the proof.
\end{proof}
\begin{lemma}
\label{lemma:errExponent}
For $n,d\in\mathbb{N}$, assume the rows of $\boldsymbol{Q}\in\left\{0,1\right\}^{n\times d}$ to be $n$ i.i.d. samples drawn from a BMM with an arbitrary parameter set. Then, the probability of observing a deviation error of $\varepsilon>0$ between $\mathcal{D}\left(\boldsymbol{Q}\right)$ and the asymptotic measure $\lim_{n\rightarrow\infty}\mathcal{D}\left(\boldsymbol{Q}\right)$ can be upper-bounded as
\begin{equation*}
\mathcal{P}\left\{
\left\vert
\mathcal{D}\left(\boldsymbol{Q}\right)-
\lim_{n\rightarrow\infty}\mathcal{D}\left(\boldsymbol{Q}\right)
\right\vert
> \varepsilon
\right\}
\leq
2^{d+1}\exp\left(
\frac{-n\varepsilon^2}{d^4 2^{d+1}}
\right).
\end{equation*}
\end{lemma}
\begin{proof}
Let us define $\hat{\mathbb{P}}_{\boldsymbol{Q}}$ as the empirical measure underlying the rows of $\boldsymbol{Q}$ (same is $\hat{\mathbb{P}}_{1,\boldsymbol{Q}}$ in Definition \ref{def:DQ}).
According to Definition \ref{def:DQ}, $\mathcal{D}\left(\boldsymbol{Q}\right)$ only depends on $\hat{\mathbb{P}}_{\boldsymbol{Q}}$, and thus permutation of the rows of $\boldsymbol{Q}$ does not affect its value. In this regard, and for the sake of simplicity, let us define $g:\mathbb{R}^{2^d}\rightarrow\mathbb{R}$ such that $g\left(\hat{\mathbb{P}}_{\boldsymbol{Q}}\right):=\mathcal{D}\left(\boldsymbol{Q}\right)$, i.e. a function that maps the empirical distribution $\hat{\mathbb{P}}_{\boldsymbol{Q}}$ to $\mathcal{D}\left(\boldsymbol{Q}\right)$. 
According to Definition \ref{def:DQ}, it can be readily verified that
\begin{align}
\label{eq:Ddef2}
\mathcal{D}\left(\boldsymbol{Q}\right)&=\sum_{\boldsymbol{X}\in\left\{0,1\right\}^d}\hat{\mathbb{P}}_{\boldsymbol{Q}}\left(\boldsymbol{X}\right)\log\left(
\frac{
\hat{\mathbb{P}}_{\boldsymbol{Q}}\left(\boldsymbol{X}\right)
}{
\prod_{\ell=1}^{d}\hat{p}_{\ell}^{X_{\ell}}\left(1-\hat{p}_{\ell}\right)^{1-X_{\ell}}
}
\right)
\nonumber\\
&=\sum_{\boldsymbol{X}\in\left\{0,1\right\}^d}
\hat{\mathbb{P}}_{\boldsymbol{Q}}\left(\boldsymbol{X}\right)
\left[
\log\hat{\mathbb{P}}_{\boldsymbol{Q}}\left(\boldsymbol{X}\right)
-
\sum_{\ell=1}^{d}
\boldsymbol{1}_{X_\ell}\log\hat{p}_\ell
+
\boldsymbol{1}_{1-X_\ell}
\log\left(1-\hat{p}_\ell\right)
\right]
\nonumber \\
&=
\sum_{\boldsymbol{X}\in\left\{0,1\right\}^d}\hat{\mathbb{P}}_{\boldsymbol{Q}}\left(\boldsymbol{X}\right)
\left[
\log\hat{\mathbb{P}}_{\boldsymbol{Q}}\left(\boldsymbol{X}\right)
-
\sum_{\ell=1}^{d}\log
\left(\sum_{\boldsymbol{X}'\in\left\{0,1\right\}^d\vert X'_\ell=X_\ell}
\hat{\mathbb{P}}_{\boldsymbol{Q}}\left(\boldsymbol{X}'\right)
\right)
\right],
\end{align}
where $\boldsymbol{1}_X$ denotes the indicator function which returns $1$ if $X=1$ and zero otherwise. During the derivation of \eqref{eq:Ddef2}, we have used the following two facts for $\ell=1,\ldots,d$:
\begin{equation*}
\hat{p}_\ell = \sum_{\boldsymbol{X}'\in\left\{0,1\right\}^d\vert X'_\ell=1}\hat{\mathbb{P}}_{\boldsymbol{Q}}\left(
\boldsymbol{X}'
\right),
\quad\quad
1-\hat{p}_\ell = \sum_{\boldsymbol{X}'\in\left\{0,1\right\}^d\vert X'_\ell=0}\hat{\mathbb{P}}_{\boldsymbol{Q}}\left(
\boldsymbol{X}'
\right).
\end{equation*}
Since $g$ is a continuous function, the {\it {law of large numbers}} implies that
\begin{equation*}
\lim_{n\rightarrow\infty}\mathcal{D}\left(\boldsymbol{Q}\right)=g\left(
\lim_{n\rightarrow\infty}\hat{\mathbb{P}}_{\boldsymbol{Q}}\right)
\stackrel{a.s.}{=}
g\left(\mathbb{P}_{\mathcal{B}}\right),
\end{equation*}
where $\mathbb{P}_{\mathcal{B}}$ represents the true distribution of the BMM $\mathcal{B}$ that underlies the rows of $\boldsymbol{Q}$. Obviously, unlike the empirical measure $\hat{\mathbb{P}}_{\boldsymbol{Q}}$, $\mathbb{P}_{\mathcal{B}}$ is a deterministic distribution which can be quantified based on the parameters of ${\mathcal{B}}$. In this regard, differential calculus implies the following relation:
\begin{align*}
\left\vert
\mathcal{D}\left(\boldsymbol{Q}\right) - \lim_{n\rightarrow\infty}\mathcal{D}\left(\boldsymbol{Q}\right)
\right\vert\stackrel{a.s.}{=}&
\left\vert
g\left(\hat{\mathbb{P}}_{\boldsymbol{Q}}\right)-g\left(\mathbb{P}_{\mathcal{B}}\right)
\right\vert=
\left\vert\int_{\mathscr{P}}\left\langle\nabla g\big\vert\mathrm{d}\mathscr{P}\right\rangle\right\vert,
\end{align*}
where $\nabla$ denotes the gradient operator, $\left\langle\cdot\vert\cdot\right\rangle$ denotes the inner product, and $\mathscr{P}$ is an arbitrary continuous path in $\mathbb{R}^{2^d}$ that starts from $\hat{\mathbb{P}}_{\boldsymbol{Q}}$ and ends in $\mathbb{P}_{\mathcal{B}}$ \footnote{In the proceeding relations, we also show that $g$ is differentiable. Therefore, $\nabla g:\mathbb{R}^{2^d}\rightarrow\mathbb{R}^{2^d}$ exists.}. Let us consider the following particular path $\mathscr{P}$: The union of $2^d$ sub-paths, where each sub-path is aligned to a distinct axis of $\mathbb{R}^{2^d}$. Thus, we move from $\hat{\mathbb{P}}_{\boldsymbol{Q}}$ to $\mathbb{P}_{\mathcal{B}}$ in $2^d$ steps, where at each step we only change one of the components and keep the rest fixed. Let us denote the above-mentioned $2^d$ sub-paths with $\mathscr{P}_{\boldsymbol{X}},~\boldsymbol{X}\in\left\{0,1\right\}^d$. In this regard, while moving along the axis that corresponds to a particular  $\boldsymbol{X}\in\left\{0,1\right\}^d$, the term $\left\langle\nabla g\vert\mathrm{d}\mathscr{P}\right\rangle$ simply becomes $\nabla_{\boldsymbol{X}}g~\mathrm{d}s$, where $s$ denotes the length parameter of the sub-path associated to component $\boldsymbol{X}$ and $\nabla_{\boldsymbol{X}}g:\mathbb{R}^{2^d}\rightarrow\mathbb{R}$ denotes the component of the $2^d$-dimensional gradient $\nabla g$ which corresponds to $\boldsymbol{X}$.

With the above specifications for $\mathscr{P}$, and using the {\it {Mean Value Theorem (MVT)}} \cite{courant2011differential}, one can write:
\begin{align*}
\left\vert
\mathcal{D}\left(\boldsymbol{Q}\right) - \lim_{n\rightarrow\infty}\mathcal{D}\left(\boldsymbol{Q}\right)
\right\vert
~\stackrel{a.s.}{\leq}~&
\sum_{\boldsymbol{X}\in\left\{0,1\right\}^{d}}
\left\vert
\int_{\mathscr{P}_{\boldsymbol{X}}}
\nabla_{\boldsymbol{X}}g
\mathrm{d}s
\right\vert
\\
\stackrel{\mathrm{(MVT)}}{\leq}&
\sum_{\boldsymbol{X}\in\left\{0,1\right\}^d}
\left(
\sup_{\boldsymbol{\nu}\in\mathscr{P}}
\left\vert
\nabla_{\boldsymbol{X}}g\left({\boldsymbol{\nu}}\right)
\right\vert
\right)
\left\vert
\hat{\mathbb{P}}_{\boldsymbol{Q}}\left(\boldsymbol{X}\right)
-\mathbb{P}_{\mathcal{B}}\left(\boldsymbol{X}\right)
\right\vert.
\end{align*}
According to \eqref{eq:Ddef2}, it is easy to show that partial derivatives of $g$ can be exactly computed at the true distribution $\mathbb{P}_{\mathcal{B}}$ through the following formula:
\begin{equation*}
\nabla_{\boldsymbol{X}}g
=
\log\left(\frac{
\mathbb{P}_{\mathcal{B}}\left(\boldsymbol{X}\right)
}{
\prod_{\ell=1}^{d}\left(\sum_{\boldsymbol{X}'\in\left\{0,1\right\}^d\vert X'_\ell=X_\ell}\mathbb{P}_{\mathcal{B}}\left(\boldsymbol{X}'\right)\right)
}
\right)-\left(d-1\right),~~ \forall\boldsymbol{X}\in\left\{0,1\right\}^{d}.
\end{equation*}
Considering the fact that $\sum_{\boldsymbol{X}'\in\left\{0,1\right\}^d\vert X'_\ell=X_\ell}\mathbb{P}_{\mathcal{B}}\left(\boldsymbol{X}'\right)\ge\mathbb{P}_{\mathcal{B}}\left(\boldsymbol{X}\right)$, the following upper-bound holds for the partial derivatives of $g$ for all $\boldsymbol{X}\in\left\{0,1\right\}^d$ and sufficiently large $n$:
\begin{align*}
\sup_{\boldsymbol{\nu}\in\mathscr{P}}
\left\vert
\nabla_{\boldsymbol{X}}g\left(\boldsymbol{\nu}\right)
\right\vert
~\leq~
\left\vert
\log\left(\frac
{\mathbb{P}_{\mathcal{B}}\left(\boldsymbol{X}\right)}
{\prod_{\ell=1}^{d}\mathbb{P}_{\mathcal{B}}\left(\boldsymbol{X}\right)}
\right)\right\vert+\left(d-1\right)
~\leq~
d\left(\log\frac{1}{\mathbb{P}_{\mathcal{B}}\left(\boldsymbol{X}\right)}+1\right).
\end{align*}
So far, we have managed to upper-bound the estimation error in the current lemma by the following inequality:
\begin{align}
\label{eq:SigmaDare}
\left\vert
\mathcal{D}\left(\boldsymbol{Q}\right) - \lim_{n\rightarrow\infty}\mathcal{D}\left(\boldsymbol{Q}\right)
\right\vert
&\stackrel{a.s.}{\leq}
d\sum_{\boldsymbol{X}\in\left\{0,1\right\}^d}
\left(\log\frac{1}{\mathbb{P}_{\mathcal{B}}\left(\boldsymbol{X}\right)}+1\right)
\left\vert
\hat{\mathbb{P}}_{\boldsymbol{Q}}\left(\boldsymbol{X}\right)
-\mathbb{P}_{\mathcal{B}}\left(\boldsymbol{X}\right)
\right\vert
\\
&=
d\sum_{\boldsymbol{X}\in\left\{0,1\right\}^d}
\left(\log\frac{1}{\mathbb{P}_{\mathcal{B}}\left(\boldsymbol{X}\right)}+1\right)
\sigma_{\boldsymbol{X}}\left\vert
\frac{
\hat{\mathbb{P}}_{\boldsymbol{Q}}\left(\boldsymbol{X}\right) - 
\mathbb{P}_{\mathcal{B}}\left(\boldsymbol{X}\right)}
{\sigma_{\boldsymbol{X}}}
\right\vert
\nonumber\\
&\leq
d\left(\max_{\boldsymbol{X}}~\left\vert
\frac{\hat{\mathbb{P}}_{\boldsymbol{Q}}\left(\boldsymbol{X}\right)
- \mathbb{P}_{\mathcal{B}}\left(\boldsymbol{X}\right)}
{\sigma_{\boldsymbol{X}}}
\right\vert
\right)
\sum_{\boldsymbol{X}\in\left\{0,1\right\}^d}
\left(\log\frac{1}{\mathbb{P}_{\mathcal{B}}\left(\boldsymbol{X}\right)}+1\right)
\sigma_{\boldsymbol{X}},
\nonumber
\end{align}
where $\sigma_{\boldsymbol{X}}:=\sqrt{\mathbb{P}_{\mathcal{B}}\left(\boldsymbol{X}\right)\left(1-\mathbb{P}_{\mathcal{B}}\left(\boldsymbol{X}\right)\right)}
$.

An important issue that should be noted is that for those cases where $\mathbb{P}_{\mathcal{B}}\left(\boldsymbol{X}\right)=0$ or $1$, we have $\sigma_{\boldsymbol{X}}=0$. However, in such cases, the empirical probabilities always coincide with the true ones, and the corresponding error terms in the above summation become exactly zero. As a result, such cases are implicitly omitted from all the summations in \eqref{eq:SigmaDare}.

It is easy to show that the summation over $\boldsymbol{X}\in\left\{0,1\right\}^d$ in the r.h.s. of \eqref{eq:SigmaDare} reaches its maximum when $\mathbb{P}_{\mathcal{B}}\left(\boldsymbol{X}\right)=2^{-d}$ for all $\boldsymbol{X}$, which means
\begin{equation*}
\sum_{\boldsymbol{X}\in\left\{0,1\right\}^d}
\left(\log\frac{1}{\mathbb{P}_{\mathcal{B}}\left(\boldsymbol{X}\right)}+1\right)
\sigma_{\boldsymbol{X}}
\leq
d2^{d/2}.
\end{equation*}
The only remaining part of the proof is to bound the difference between the true distribution $\mathbb{P}_{\mathcal{B}}$ and the empirical one $\hat{\mathbb{P}}_{\boldsymbol{Q}}$. Let us define the set of events $A_{\boldsymbol{X}}~,~\forall \boldsymbol{X}\in\left\{0,1\right\}^d$ as
\begin{equation*}
A_{\boldsymbol{X}}~:= ~\left\vert
\frac{
\hat{\mathbb{P}}_{\boldsymbol{Q}}\left(\boldsymbol{X}\right) - \mathbb{P}_{\mathcal{B}}\left(\boldsymbol{X}\right)}
{\sigma_{\boldsymbol{X}}}
\right\vert > \delta,
\quad\mathrm{where}~\delta:=\frac{\varepsilon}{d^2 2^{d/2}}.
\end{equation*}
Based on the previous relations, it can be verified that if none of the events $A_{\boldsymbol{X}}$ occur, then we almost surely have $\left\vert
\mathcal{D}\left(\boldsymbol{Q}\right) - \lim_{n\rightarrow\infty}\mathcal{D}\left(\boldsymbol{Q}\right)
\right\vert
\leq\varepsilon$. Therefore, for $\varepsilon>0$, and using both the Union Bound (UB) and Chernoff Bound (CB), one can show
\begin{align*}
\mathbb{P}&\left\{
\left\vert
\mathcal{D}\left(\boldsymbol{Q}\right) - \lim_{n\rightarrow\infty}\mathcal{D}\left(\boldsymbol{Q}\right)
\right\vert
>\varepsilon
\right\}
\leq \mathbb{P}\left\{
\bigcup_{\boldsymbol{X}\in\left\{0,1\right\}^d}
A_{\boldsymbol{X}}
\right\}
\stackrel{UB}{\leq}\sum_{\boldsymbol{X}\in\left\{0,1\right\}^d}\mathbb{P}\left\{
A_{\boldsymbol{X}}
\right\}
\\
&\hspace{35mm}
\stackrel{CB}{\leq}
\sum_{\boldsymbol{X}\in\left\{0,1\right\}^d}
e^{-n\mathcal{D}_{\mathrm{KL}}\left(
\mathbb{P}_{\mathcal{B}}\left(\boldsymbol{X}\right)+\delta\sigma_{\boldsymbol{X}}
\big\Vert
\mathbb{P}_{\mathcal{B}}\left(\boldsymbol{X}\right)
\right)
}
+\sum_{\boldsymbol{X}\in\left\{0,1\right\}^d}
e^{-n\mathcal{D}_{\mathrm{KL}}\left(
\mathbb{P}_{\mathcal{B}}\left(\boldsymbol{X}\right)-\delta\sigma_{\boldsymbol{X}}
\big\Vert
\mathbb{P}_{\mathcal{B}}\left(\boldsymbol{X}\right)\right)
},
\end{align*}
where $\mathcal{D}_{\mathrm{KL}}\left(\cdot\Vert\cdot\right)$ represents the Kullback-Leibler divergence, and by $\mathcal{D}_{\mathrm{KL}}\left(x\Vert y\right)$ for $x,y\in\left[0,1\right]$ we mean
\begin{equation*}
x\log\frac{x}{y} + \left(1-x\right)\log\frac{1-x}{1-y}.
\end{equation*}
For $x\notin\left[0,1\right]$, let us define $\mathcal{D}_{\mathrm{KL}}\left(x\Vert y\right):=+\infty$.

KL divergence can be lower-bounded according to Chernoff's theorem \cite{diakonikolas2016learning}. In other words, we have
\begin{align*}
\mathbb{P}\left\{
\left\vert
\mathcal{D}\left(\boldsymbol{Q}\right) - \lim_{n\rightarrow\infty}\mathcal{D}\left(\boldsymbol{Q}\right)
\right\vert
>\varepsilon
\right\}
&\leq
2\cdot 2^d \cdot \max_{\boldsymbol{X}\in\left\{0,1\right\}^d}\max_{\theta\in\left\{-1,+1\right\}}
e^{-n\mathcal{D}_{\mathrm{KL}}\left(
\mathbb{P}_{\mathcal{B}}\left(\boldsymbol{X}\right)+\theta\delta\sigma_{\boldsymbol{X}}
\big\Vert
\mathbb{P}_{\mathcal{B}}\left(\boldsymbol{X}\right)\right)
}
\\
&\leq
2^{d+1}
\max_{\boldsymbol{X},\theta}~
\exp\left(
\frac{-n\delta^2\theta^2\sigma^2_{\boldsymbol{X}}}{2\mathbb{P}_{\mathcal{B}}\left(\boldsymbol{X}\right)\left(1-
\mathbb{P}_{\mathcal{B}}\left(\boldsymbol{X}\right)
\right)}
\right).
\end{align*}
By substituting for $\delta$ and considering the definition of $\sigma_{\boldsymbol{X}}$, the probability of observing a deviation greater than $\varepsilon$ in estimating $\lim_{n\rightarrow\infty}\mathcal{D}\left(\boldsymbol{Q}\right)$ can be upper-bounded as
\begin{equation*}
\mathcal{P}\left\{
\left\vert
\mathcal{D}\left(\boldsymbol{Q}\right)-
\lim_{n\rightarrow\infty}\mathcal{D}\left(\boldsymbol{Q}\right)
\right\vert
> \varepsilon
\right\}
\leq
2^{d+1}\exp\left(
\frac{-n\varepsilon^2}{d^4 2^{d+1}}
\right),
\end{equation*}
which completes the proof.
\end{proof}
\begin{proof}[Proof of Lemma \ref{lemma:errExponent2}]
The proof is highly similar to that of Lemma \ref{lemma:errExponent}. The main difference lies in the fact that when $K=1$, i.e. a single Bernoulli model, one can easily verify that for all $\boldsymbol{X}\in\left\{0,1\right\}^{d}$, we have:
\begin{gather*}
\nabla_{\boldsymbol{X}}g=
\log\left(\frac
{\mathbb{P}_{\mathcal{B}}\left(\boldsymbol{X}\right)}
{\prod_{\ell=1}^{d}
\left(\sum_{\boldsymbol{X}'\in\left\{0,1\right\}^d\vert X'_\ell=X_\ell}{\mathbb{P}_{\mathcal{B}}}\left(\boldsymbol{X}'\right)\right)}
\right)-\left(d-1\right)=1-d,
\end{gather*}
since in a single Bernoulli model, the probability distribution equals to the product of its marginals over each dimension. Therefore, we have $\left\vert \nabla_{\boldsymbol{X}}g \right\vert\leq d$. Following the same steps as shown in the proof of Lemma \ref{lemma:errExponent} gives us the claimed inequality and complete the proof.
\end{proof}

\begin{proof}[Proof of Lemma \ref{thm:corl1-2}]
Recall $\mathrm{Col}\left(\boldsymbol{Y};d\right)$ as the set of all $\binom{L}{d}$ sub-matrices of $\boldsymbol{Y}$ with $d$ columns. Then, the first inequality states that the probability of  $\exists \boldsymbol{Q}\in\mathrm{Col}\left(\boldsymbol{Y};d\right)\Rightarrow\mathcal{D}\left(\boldsymbol{Q}\right)<\tau$ is strictly bounded.

In the following, we show that by examining all $\binom{L}{d}$ sub-matrices in $\mathrm{Col}\left(\boldsymbol{Y};d\right)$, one can find at least $h:=\left\lfloor \mathcal{L}/d\right\rfloor$ disjoint column sub-matrices of $\boldsymbol{Y}$, denoted by $\boldsymbol{Q}_1,\ldots,\boldsymbol{Q}_h$, such that the frequency sub-matrices that correspond to $Q_i$s are guaranteed to be at least $\left(\left\lfloor 
\frac{2d}{K\left(K-1\right)}
\right\rfloor,\delta\right)$-separable: 
\\[1mm]
\noindent
First, it should be noted that frequency matrix $\boldsymbol{P}$ is assumed to be $\left(\mathcal{L},\delta\right)$-separable. Similar to the notation we used in the proof of Lemma \ref{lemma:Main}, it can said that for all pairs of rows in $\boldsymbol{P}$, say $i$ and $j$, there exists a subset of columns $\mathscr{C}_{i,j}\subseteq\left\{1,2,\ldots,L\right\}$, where
\begin{equation*}
\left\vert p^{\left(i\right)}_\ell - p^{\left(j\right)}_\ell \right\vert \ge \delta~,~\ell\in\mathscr{C}_{i,j},
\end{equation*}
and $\left\vert \mathscr{C}_{i,j}\right\vert \ge \mathcal{L}$. For each $i=1,\ldots,K-1$, let us take $\left\lfloor 2d/\left[K(K-1)\right] \right\rfloor$ arbitrarily chosen indices from each of the $K-i$ sets $\mathscr{C}_{i,j},~j>i$ and then put them in some new corresponding sets, denoted by $\mathscr{D}_{i,j},~j>i$. It should be noted that $\mathscr{D}_{i,j}$s may have non-empty overlaps. Let 
$$
\mathscr{D}:=\bigcup_{j>i}\mathscr{D}_{i,j}.
$$
Obviously, $\mathscr{D}$ cannot have more than $d$ members, since its the union of $K\left(K-1\right)/2$ sets, each having $\left\lfloor 2d/\left[K(K-1)\right] \right\rfloor$ members. In many practical situations, {\it {informative dimensions}} are dispersed randomly and thus $\mathscr{D}_{i,j}$s can be chosen to have huge overlaps. However, we consider the worst case which assumes the overlaps are empty. Also, for the cases where $\left\vert\mathscr{D}\right\vert<d$, assume we add enough arbitrary indices to $\mathscr{D}$ until it has $d$ members. In this regard, the indices in $\mathscr{D}$ correspond to a sub-matrix of frequency matrix $\boldsymbol{P}$ that is at least $\left(\left\lfloor 2d/\left[K(K-1)\right] \right\rfloor,\delta\right)$-separable.

On the other hand, we can repeat the above procedure for at least $h:=\left\lfloor\mathcal{L}/d \right\rfloor$ times without choosing any dimension more than once. This results in at least $h$ disjoint sub-matrices, called $\boldsymbol{Q}_1,\ldots,\boldsymbol{Q}_h$, that possess the above-mentioned property. Since $\boldsymbol{Q}_1,\ldots,\boldsymbol{Q}_h$ do not overlap with each other, they are statistically independent which then implies
\begin{align*}
\mathbb{P}\left\{\mathcal{D}_{\max}\left(\boldsymbol{Y},d\right)\leq \tau\right\}
\leq
\mathbb{P}\left\{\mathcal{D}\left(\boldsymbol{Q}_i\right)\leq \tau~,~\forall i\right\}
=
\prod_{i=1}^{h}\mathbb{P}\left\{\mathcal{D}\left(\boldsymbol{Q}_i\right)\leq \tau\right\}.
\end{align*}
Using the upper-bound for each $\mathbb{P}\left\{\mathcal{D}\left(\boldsymbol{Q}_i\right)\leq \tau\right\}$ from Lemma \ref{lemma:Main} and approximating $h$ with $\mathcal{L}/d$, one can simply prove the claimed inequality.

For the second inequality in the statement of Lemma \ref{thm:corl1-2}, one can simply employ the union bound as follows:
\begin{align}
\mathbb{P}\left\{
\mathcal{D}_{\max}\left(\boldsymbol{Y};d\right)>\tau
\right\}
\leq
\mathbb{P}\left\{
\max_{\boldsymbol{Q}\in\mathrm{Col}\left(\boldsymbol{Y};d\right)}\mathcal{D}\left(\boldsymbol{Q}\right)>\tau
\right\}
\leq
\sum_{\boldsymbol{Q}\in\mathrm{Col}\left(\boldsymbol{Y};d\right)}
\mathbb{P}\left\{
\mathcal{D}\left(\boldsymbol{Q}\right)>\tau
\right\}.
\label{eq:corl2proof}
\end{align}
Also, note that $\mathrm{Col}\left(\boldsymbol{Y};d\right)$ includes $\binom{L}{d}$ members. Again, substitution of $\mathbb{P}\left\{\mathcal{D}\left(\boldsymbol{Q}\right)>\tau\right\}$ with the upper-bound derived in Lemma \ref{lemma:errExponent} gives us the claimed inequality and completes the proof.
\end{proof}

\begin{lemma}
\label{lemma:minClusterMember}
Consider $\mathcal{M}=\mathcal{M}\left(K,\boldsymbol{w}\right)$ to be a multinomial distribution with $K$ mutually exclusive outcomes and corresponding probability vector $\boldsymbol{w}=\left(w_1,\ldots,w_K\right)$. Assume there exists $\alpha>0$ such that $\min_k w_k\ge\alpha$. Let $\boldsymbol{D}:=\left\{X_1,\ldots,X_n\right\}$ to be $n$ i.i.d. samples drawn from $\mathcal{M}$. For $\zeta>0$, assume
\begin{equation*}
n\ge \frac{2}{\alpha^2}\log\frac{3K}{\zeta}.
\end{equation*}
Then, with probability at least $1-\zeta/3$, the size of the smallest cluster in $\boldsymbol{D}$ is least $\alpha n/2$.
\end{lemma}
\begin{proof}
We denote the probability of the smallest cluster in $\boldsymbol{D}$ having less than $\alpha n/2$ members by $P_E$. Let $\mathcal{A}_1,\ldots,\mathcal{A}_K$ represent the following events: for $k=1,\ldots,K$, $\mathcal{A}_k$ represents the event that the $k$th cluster in $\boldsymbol{D}$ (corresponding to probability component $w_k$) has less than $nw_k/2$ members. Then, the following holds according to union bound:
\begin{equation}
P_E\leq
\sum_{k=1}^{K}\mathbb{P}\left\{\mathcal{A}_k\right\}.
\end{equation}
For $k=1,\ldots,K$, consider the binomial random variable $Y_k$ with the following distribution:
\begin{equation}
\mathbb{P}\left(Y_k\right):=
\left\{\begin{array}{lc}
w_k     &  Y_k=1
\\
1-w_k     & Y_k=0
\end{array}
\right.,
\end{equation}
with $\mathbb{E}Y_k=w_k$. Let $y_1,\ldots,y_n$ to be $n$ i.i.d. samples of $Y_k$. Define $S_k:= y_1+\cdots+y_n$, while obviously we have $\mathbb{E}S_k=nw_k$.
Using Hoeffding's inequality, one can easily verify the following chain of relations:
\begin{align}
\mathbb{P}\left\{\mathcal{A}_k\right\}=&
\mathbb{P}\left\{S_k < nw_k/2\right\}=
\mathbb{P}\left\{S_k - \mathbb{E}S_k< -nw_k/2\right\}
\nonumber\\
\leq&
\exp\left(\frac{-2n^2 w^2_k}{4\sum_{i=1}^{n}\left(\max Y_k-\min Y_k\right)^2}\right)
\nonumber\\
=& \exp\left(\frac{-nw^2_k}{2}\right).
\end{align}
Recall that we have $w_k\ge\alpha$ for all $k=1,\ldots,K$, thus one can write
\begin{equation}
P_E \leq \sum_{k=1}^{K} \exp\left(\frac{-nw^2_k}{2}\right)
\leq
K\exp\left(\frac{-n\alpha^2}{2}\right),
\end{equation}
which given the condition on $n$ in the lemma, results into $P_E\leq \zeta/3$ and completes the proof.
\end{proof}

\end{document}